\documentclass[final]{colt-plain}

\usepackage{amsmath}
\usepackage{amssymb}
\usepackage{mathtools}
\usepackage{enumerate}
\usepackage{pgfplots}
\usepackage{pgfplotstable}
\usepackage{breakurl}
\usepackage{dsfont}
\usepackage{enumitem}
\usepackage{wrapfig}
\usepackage[boxed]{algorithm}
\usepackage{algpseudocode}
\usepackage{booktabs}
\usepackage{float}
\usepackage{natbib}
\usepackage{url}
\usepackage[capitalize]{cleveref}

\usepackage{times}
\usepackage[font={bf,small}]{caption}

\newenvironment{proofof}[1]%
{%
 \par\noindent{\bfseries\upshape Proof\ of\ #1\ }%
}%
{\jmlrQED}

\newcommand{\E}{\mathbb E}

\newcommand{\sr}[1]{\stackrel{#1}}
\renewcommand{\set}[1]{\left\{#1\right\}}
\newcommand{\ind}[1]{\mathds{1}\!\!\set{#1}}
\newcommand{\argmax}{\operatornamewithlimits{arg\,max}}
\newcommand{\ceil}[1]{\left \lceil {#1} \right\rceil}
\newcommand{\eqn}[1]{\begin{align}#1\end{align}}
\newcommand{\eq}[1]{\begin{align*}#1\end{align*}}
\newcommand{\logp}{\log_{+}\!}
\newcommand{\BR}{B\!R}
\newcommand{\diag}{\operatorname{diag}}
\newcommand{\ucb}{\text{\scalebox{0.8}{UCB}}}
\newcommand{\gittins}{\text{\scalebox{0.8}{Gittins}}}

\renewcommand{\P}[1]{\mathbb{P}\left\{#1\right\}}

\newcommand{\plog}{\operatorname{W\!}}

\newcommand{\R}{\mathbb R}
\newcommand{\N}{\mathbb N}

\newcommand{\erfc}{\operatorname{Erfc}}
\newcommand{\erf}{\operatorname{Erf}}

\newtheorem{assumption}[theorem]{Assumption}

\let\epsilon\varepsilon

\def\subsubsect#1{\vspace{1ex plus 0.5ex minus 0.5ex}\noindent{\normalsize\textbf{#1\ }}}

\newcommand{\thetitle}{{Regret Analysis of the Finite-Horizon Gittins Index Strategy for Multi-Armed Bandits}}
\newcommand{\theruntitle}{{Regret Analysis of the Finite-Horizon Gittins Index Strategy}}

\newcommand{\theabstract}{
I prove near-optimal frequentist regret guarantees for the finite-horizon Gittins index strategy for multi-armed bandits with Gaussian noise and prior.
Along the way I derive finite-time bounds on the Gittins index that are asymptotically exact and
may be of independent interest. 
I also discuss computational issues and present experimental results suggesting that a particular version of the Gittins index strategy is an 
improvement on existing algorithms with finite-time regret guarantees such as UCB and Thompson sampling. 
}

\begin{document}

\title[\theruntitle]{\thetitle}
 \coltauthor{\Name{Tor Lattimore} \Email{tor.lattimore@gmail.com}\\
 \addr Department of Computing Science \\
 University of Alberta \\
 Edmonton, Canada
 }
\maketitle

\begin{abstract}
\theabstract
\end{abstract}

\section{Introduction}

The stochastic multi-armed bandit is a classical problem in sequential optimisation that captures a particularly interesting aspect of the dilemma faced by
learning agents. How to explore an uncertain world, while simultaneously exploiting available information?
Since \cite{Rob52} popularised the problem there have been two main solution concepts. The first being the Bayesian approach developed by 
\cite{BJK56}, \cite{Git79} and others, where research has primarily focussed on characterising optimal solutions. The second approach is frequentist,
with the objective of designing policies that minimise various forms of regret \citep{LR85}.
The purpose of this article is to prove frequentist regret guarantees for popular Bayesian or near-Bayesian algorithms, which explicates
the strong empirical performance of these approaches observed by \cite{KCOG12} and others.

In each round the learner chooses an arm $I_t \in \set{1,\ldots,d}$ based on past observations and receives a Gaussian reward $X_t \sim \mathcal N(\mu_{I_t}, 1)$
where $\mu \in \R^d$ is the vector of unknown means.  
A strategy is a method for choosing $I_t$ and is denoted by $\pi$.
The performance of a particular strategy $\pi$ will be measured in terms of the expected regret, which measures the difference
between the expected cumulative reward of the optimal strategy that knows $\mu$ and the expected rewards of $\pi$. Let $\mu^* = \max_i \mu_i$ be
the mean reward for the optimal arm, then the expected regret up to horizon $n$ is defined by
\eqn{
\label{eq:regret}
R^\pi_\mu(n) = \E\left[\sum_{t=1}^n (\mu^* - \mu_{I_t})\right]\,,
}
where the expectation is taken with respect to the uncertainty in the rewards and any randomness in the strategy.
If $Q$ is a probability measure on $\R^d$, then the Bayesian regret is the expectation of \cref{eq:regret} with respect to the prior $Q$.
\eqn{
\label{eq:bayes}
\BR^\pi_Q(n) = \E_{\theta \sim Q} \! \left[\E \left[\sum_{t=1}^n (\mu^* - \mu_{I_t}) \Bigg| \mu = \theta\right]\right]\,.
}
I assume that $Q$ is Gaussian with diagonal covariance matrix $\Sigma = \diag(\sigma_1^2,\ldots,\sigma_d^2)$.

A famous non-Bayesian strategy is UCB \citep{KR95,Agr95,ACF02}, which chooses $I_t = t$ for rounds $t \in \set{1,\ldots,d}$ and subsequently
maximises an upper confidence bound.
\eq{
I_t = \argmax_i \hat \mu_i(t-1) + \sqrt{\frac{\alpha \log t}{T_i(t-1)}}\,,
}
where $\hat \mu_i(t-1)$ is the empirical estimate of the return of arm $i$ based on samples from the first $t-1$ rounds 
and $T_i(t-1)$ is the number of times that
arm has been chosen. For $\alpha > 2$ and any choice of $\mu$ it can be shown that
\eqn{
\label{eq:ucb}
R^{\ucb}_\mu(n) = O\left( \sum_{i: \Delta_i > 0} \frac{\log(n)}{\Delta_i} + \Delta_i\right)\,,
}
where $\Delta_i = \mu^* - \mu_i$ is the regret incurred by choosing arm $i$ rather than the optimal arm.
No strategy enjoying sub-polynomial regret for all mean 
vectors can achieve smaller asymptotic regret than \cref{eq:ucb}, so in this sense the UCB strategy is optimal \citep{LR85}.

The Bayesian strategy minimises \cref{eq:bayes}, which appears to be a hopeless optimisation problem.
A special case where it can be solved efficiently is called the one-armed bandit, which occurs when there
are two arms ($d = 2$) and the expected return of the second arm is known ($\sigma^2_2 = 0$).
\cite{BJK56} showed that the Bayesian optimal strategy involves choosing the first arm as long as its ``index'' is larger than
the return of the second arm and thereafter choosing the second arm. The index depends only on the number of rounds remaining
and the posterior distribution of the first arm and is computed by solving an optimal stopping problem. Another 
situation when \cref{eq:bayes} can be solved efficiently is when the horizon is 
infinite ($n = \infty$) and the rewards are discounted geometrically. Then \cite{Git79} was able to show that
the Bayesian strategy chooses in each round the arm with maximal index. Gittins' index is defined in the same fashion as the index of \cite{BJK56}
but with obvious modifications to incorporate the discounting.
The index has a variety of interpretations. For example, it is 
equal to the price per round that a rational learner should be willing to pay in order to play
the arm until either the horizon or their choice of stopping time \citep{Web92}.

Gittins' result is remarkable because it reduces the seemingly intractable problem
of finding the Bayesian optimal strategy to solving an optimal stopping problem for each
arm separately. 
In our setting, however, the horizon is finite and the rewards are not discounted, which means that Gittins' result does not apply and the Bayesian 
optimality of Gittins index strategy is not preserved.\footnote{
This contradicts the opposite claim made without proof by \cite{KCOG12} but counter-examples for Bernoulli noise have been known for some time
and are given here for Gaussian noise in \cref{sec:bayes}. In fact, the Gittins index strategy is only Bayesian optimal in all generality for
geometric discounting \citep[\S6]{BF85}.}
Nevertheless, the finite-horizon Gittins index strategy has been suggested as a tractable heuristic for the Bayesian optimal policy by
\cite{Nin11} and \cite{Kau16}, a claim that I support with empirical and theoretical results in the special case that $d = 2$ (with both means unknown).
A brief chronological summary of the literature on multi-armed bandits as relevant for this article is given in \cref{sec:history}.

I make a number of contributions. 

\subsubsect{Contribution 1 \normalfont (\cref{sec:approx})}
Upper and lower bounds on the finite-horizon Gittins index for the Gaussian prior/noise model 
that match asymptotically and are near-optimal in finite time.
Asymptotic approximations are known for the discounted case via an elegant embedding of the discrete stopping problem into
continuous time and solving the heat equation as a model for Brownian motion \citep{Bat83,Yao06}. 
In the finite-horizon setting there are also known approximations, but again
they are asymptotic in the horizon and are not suitable for regret analysis \citep{CR65,BK97}.

\subsubsect{Contribution 2 \normalfont (\cref{sec:finite})}
A proof that the Gittins index strategy with the improper flat Gaussian prior enjoys finite-time regret guarantees comparable to those of UCB.
There exists a universal constant $c > 0$ such that
\eqn{
\label{eq:log-regret}
R^{\gittins}_\mu(n) \leq  c\cdot \left( \sum_{i : \Delta_i > 0} \frac{\log(n)}{\Delta_i} + \Delta_i \right)\,.
}
I also show that the Gittins index strategy is asymptotically order-optimal for arbitrary Gaussian priors. 
There exists a universal constant $c > 0$ such that 
for all non-degenerate Gaussian priors the Bayesian strategy satisfies
\eq{
\limsup_{n\to\infty} \frac{R^{\gittins}_\mu(n)}{\log (n)} \leq  c \cdot \left(\sum_{i:\Delta_i > 0} \frac{1}{\Delta_i}\right) \,.
}
While \cref{eq:log-regret} depends on a particular choice of prior, the above result is asymptotic, but independent of the prior.
This is yet another example of a common property of (near) Bayesian methods, which is that the effect of the prior is diluted in the limit
of infinite data. A unique aspect of the bandit setting is that the strategy is collecting its own data, so although the result is asymptotic, one
has to prove that the strategy is choosing the optimal arm sufficiently often to get started.

\subsubsect{Contribution 3 \normalfont (\cref{sec:bayes})}
For the special case that there are two arms ($d = 2$) I show by a comparison to the Gittins index strategy that the fully Bayesian strategy (minimising \cref{eq:bayes})
also enjoys the two regret guarantees above.

\subsubsect{Contribution 4 \normalfont (\cref{sec:compute,sec:exp})}
I present a method of computing the index in the Gaussian case to arbitrary precision and with sufficient efficiency to simulate the Gittins index policy 
for horizons of order $10^4$. This is used to 
demonstrate empirically that the Gittins index strategy with a flat prior is competitive with state-of-the-art algorithms including UCB,
Thompson sampling \citep{Tho33}, OCUCB \citep{Lat15-ucb} and the fully Bayesian strategy.
I propose an efficient approximation of the Gittins index that (a) comes with
theoretical guarantees and (b) is empirically superior to Thompson sampling and UCB, and only marginally worse than OCUCB. The approximation may
also lead to improvements for index-based algorithms in other settings.

\section{Bounds on the Gittins Index}\label{sec:approx}
As previously remarked, the Gittins index depends only on the posterior mean and variance for the relevant arm and also on the number of
rounds remaining in round $t$, which is denoted by $m = n - t + 1$.
Let $\nu \in \R$ and $\sigma^2 \geq 0$ be the posterior mean and variance for some arm in a given round. 
Let $\mu \sim \mathcal N(\nu, \sigma^2)$ and
$Y_1, Y_2, \ldots, Y_m$ be a sequence of random variables with $Y_t \sim \mathcal N(\mu, 1)$.
Assume $\set{Y_t}_{t=1}^m$ are independent after conditioning on $\mu$.
The Gittins index for an arm with posterior mean $\nu$ and variance $\sigma^2$ and horizon $m = n-t+1$ is given by
\eqn{
\label{eq:gittins}
\gamma(\nu, \sigma^2, m) = \max\set{\gamma : \sup_{1 \leq \tau \leq m} \E\left[(\mu - \gamma) \tau\right] \geq 0}\,,
}
where the supremum is taken over stopping times with respect to the filtration generated by random variables $Y_1,\ldots, Y_m$. Equivalently,
$\ind{\tau = t}$ must be measurable with respect to $\sigma$-algebra generated by $Y_1,\ldots,Y_t$. 
I denote $\tau(\nu, \sigma^2, m)$ to be the maximising stopping time in \cref{eq:gittins} for $\gamma = \gamma(\nu,\sigma^2, m)$.
The following upper and lower bound on $\gamma$ is the main theorem of this section.

\begin{theorem}\label{thm:gittins}
If $\beta \geq 1$ is chosen such that 
$\gamma(\nu, \sigma^2, m) = \nu + \sqrt{2 \sigma^2 \log \beta}$.
Then there exists a universal constant $c > 0$ such that
\eq{
c \min\set{\frac{m}{\log_{+}^{\frac{3}{2}}(m)},\, \frac{m \sigma^2}{\log_{+}^{\frac{1}{2}}(m\sigma^2) }} 
\leq \beta 
\leq \frac{m}{\log^{\frac{3}{2}}(m)}\,,
}
where $\logp(x) = \max\set{1, \log(x)}$.
\end{theorem}

\subsubsect{Important Remark} I use $c, c'$ and $c''$ as \textit{temporary} positive constants that change from theorem to theorem and proof to proof. A table
of notation may be found in \cref{sec:notation}. \\

\noindent The upper bound in \cref{thm:gittins} is rather trivial while the lower bound relies on a carefully constructed stopping time.
As far as the asymptotics are concerned,
as $m$ tends to infinity the upper and lower bounds on $\beta$ converge up to constant factors, which leads to
\eqn{
\label{eq:asy}
\exp\left(\left(\frac{\gamma(\nu, \sigma^2, m) - \nu}{\sqrt{2\sigma^2}}\right)^2\right) = \Theta\left(\frac{m}{\log^{\frac{3}{2}} m}\right)\,.
}
The $\log^{\frac{3}{2}}(m)$ looks bizarre, but is surprisingly well-justified as shall be seen in the next section.
In order to obtain logarithmic regret guarantees it is necessary that the approximation be accurate in the exponential scale as in \cref{eq:asy}.
Merely showing that $\log \beta / \log(m)$ tends to unity as $m$ grows would not be sufficient.
Empirically choosing $c = 1/4$ in the lower bound leads to an excellent approximation of the Gittins index (see \cref{fig:approx-index} in Appendix \ref{sec:exp}).
The proof of \cref{thm:gittins} relies on a carefully chosen stopping time.
\eqn{
\label{eq:stopping}
\tau = \min\set{m,\,\, \min\set{t \geq \frac{1}{\nu^2} : \hat \mu_t + \sqrt{\frac{4}{t} \log\left(4t \nu^2 \right)} \leq 0}}\,,
}
where $\hat \mu_t = \frac{1}{t} \sum_{s=1}^t Y_s$. Note that $\ind{\tau = t}$ depends only on $Y_1,Y_2,\ldots, Y_t$ so this is a stopping time with
respect to the right filtration.

\begin{lemma}\label{lem:bandit}
If $\nu < 0$, then there exists a universal constant $c' \geq 2e$ such that:
\begin{enumerate}
\item If $\theta \in (-\infty, \nu]$, then 
$\displaystyle \E[\tau|\mu = \theta] \leq 1 + \frac{c'}{\nu^2}$.
\item If $\theta \in (\nu, 0)$, then
$\displaystyle \E[\tau|\mu = \theta] \leq 1 + \frac{c'}{\theta^2} \log\left(\frac{e\nu^2}{\theta^2}\right)$.
\item If $\theta \in [0,\infty)$, then $\E[\tau|\mu = \theta] \geq m/2$.
\end{enumerate}
\end{lemma}

The proof of Lemma \ref{lem:bandit} may be found in \cref{sec:lem:bandit}.
Before the proof of \cref{thm:gittins} we need two more results. The proofs follow directly
from the definition of the Gittins index and are omitted.

\begin{lemma}\label{lem:shift}
For all $m \geq 1$ and $\sigma^2 \geq 0$ and $\nu, \nu' \in \R$ we have
$\gamma(\nu, \sigma^2, m) - \nu = \gamma(\nu',\sigma^2, m) - \nu'$.
\end{lemma}

\begin{lemma}\label{lem:positive}
$\gamma(\nu, \sigma^2, m) \geq \nu$ for all $\nu$, $\sigma^2$ and $m$.
\end{lemma}

\begin{proofof}{\cref{thm:gittins}}
Starting with the lower bound,
by Lemma \ref{lem:shift} it is enough to bound the Gittins index for any choice of $\nu$.
Let $\nu = -\sqrt{2\sigma^2 \log \beta}$ where $\beta \geq 1$ will be defined subsequently.
I will shortly show that there exists a stopping time $\tau$ adapted to the filtration generated by $Y_1,\ldots,Y_m$ for which
\eqn{
\label{eq:zero}
\E\left[\mu\tau\right] \geq 0\,.
}
Therefore by the definition of the Gittins index we have $\gamma(\nu, \sigma^2, m) \geq 0$ and so by Lemma \ref{lem:shift} 
\eq{
(\forall \nu') \qquad \gamma(\nu', \sigma^2, m) \geq \nu' + \sqrt{2\sigma^2 \log \beta}\,.
}
I now prove \cref{eq:zero} holds for the stopping time given in \cref{eq:stopping}. First note that if $\beta = 1$, then $\nu = 0$
and $\E\left[\mu \tau\right] = 0$ for the stopping time with $\tau = 1$ surely. Now assume $\beta > 1$ and  
define $P = \mathcal N(\nu, \sigma^2)$ to be a Gaussian measure on $\R$. Then 
\eqn{
\E\left[\mu\tau\right] = \E_{\theta \sim P} \left[\theta \E \left[\tau|\mu = \theta\right]\right] 
&= \int^\infty_0 \theta \E\left[\tau|\mu = \theta\right] dP(\theta) 
  +  \int^0_{-\infty} \theta \E\left[\tau|\mu = \theta\right]  dP(\theta)\,. 
\label{eq:split}
}
The integrals will be bounded seperately. Define
\eq{
\erfc(x) &= \frac{2}{\sqrt{\pi}} \int^\infty_x \exp\left(-y^2\right)dy  & 
f(\beta) &= \frac{1}{\beta} \sqrt{\frac{1}{2\pi}} - \sqrt{\frac{\log \beta}{2}}\erfc\left(\sqrt{\log \beta}\right)\,.
}
A straightforward computation combined with Lemma \ref{lem:bandit} shows that
\eqn{
\label{eq:first}
&\int^\infty_0 \theta \E\left[\tau|\mu = \theta\right]  dP(\theta) 
\geq \frac{m}{2} \int^\infty_0 \theta dP(\theta) 
= \frac{m\sigma}{2} f(\beta)\,.
}
The function $f$ satisfies $f(\beta) \beta \logp(\beta) = \Theta(1)$ with the precise statement given in Lemma \ref{lem:beta} in \cref{sec:algebra}.
Let $W:[0,\infty) \to \R$ be the product logarithm, defined implicitly such that $x = W(x) \exp(W(x))$.
For the second integral in \cref{eq:split} we can use Lemma \ref{lem:bandit} again, which for $\theta \in (\nu, 0)$ gives
\eq{
\theta\left(\E[\tau|\mu = \theta] - 1\right)
&\geq \theta \min\set{m,\, \frac{c'}{\theta^2} \log\left(\frac{e\nu^2}{\theta^2}\right)} 
\geq -\sqrt{c' m \plog\left(\frac{e m\nu^2}{c'}\right)}\,,
}
where in the first inequality we have exploited the fact that $\tau \leq m$ occurs surely and the second follows by noting that
$\theta$ is negative and that $\theta m$ is increasing in $\theta$ and $1/\theta \log(e \nu^2 \theta^{-2})$ is decreasing.
Let $\epsilon = 1 - \sqrt{7/8}$, which is chosen such that by tail bounds on the Gaussian integral (Lemma \ref{lem:gaussian}): 
\eq{
\int^0_{\epsilon \nu} dP(\theta) 
\leq \exp\left(-\frac{(1 - \epsilon)^2 \nu^2}{2\sigma^2}\right) 
= \left(\frac{1}{\beta}\right)^{\frac{7}{8}}
\quad\text{ and }\quad
\int^{0}_{-\infty} \theta dP(\theta) 
\geq \nu - \frac{\sigma}{\sqrt{2\pi}}\,. 
}
Combining the above two displays with Lemma \ref{lem:bandit} we have for some universal constant $c'' > 0$
\eq{
\int^0_{-\infty} \theta &\E[\tau|\mu = \theta] dP(\theta)  
\geq \left(1 + \frac{c'}{\epsilon^2 \nu^2} \log\left(\frac{e}{\epsilon^2}\right)\right) \int^{0}_{-\infty} \theta dP(\theta) - \int^0_{\epsilon\nu} \sqrt{c'm \plog\left(\frac{e m\nu^2}{c'} \right)} dP(\theta) \\
&\sr{(a)}\geq c''\left(\left(1 + \frac{1}{\nu^2}\right) \left(\nu - \sigma\right) - \sqrt{m \plog\left(\frac{e m\nu^2}{c'} \right)} \left(\frac{1}{\beta}\right)^{\frac{7}{8}}\right) \\
&\sr{(b)}\geq -c''\left(\left(1 + \frac{1}{2\sigma^2 \log(\beta)}\right) \left(\sqrt{2\sigma^2 \log(\beta)} + \sigma\right) + \sqrt{m \plog\left(m\sigma^2 \log(\beta) \right)} \left(\frac{1}{\beta}\right)^{\frac{7}{8}}\right)\,,
}
where in (a) I have substituted the bounds on the integrals in the previous two displays and naively chosen a large constant $c''$ to simplify the expression and in (b) I substitute $\nu=-\sqrt{2\sigma^2 \log(\beta)}$ and
exploited the assumption that $c' \geq 2e$.
Therefore by \cref{eq:first} and \cref{eq:split} we have
\eq{
\E[\mu\tau] 
\geq \frac{m\sigma}{2} f(\beta) 
- c''\left(\left(1 \!+\! \frac{1}{2\sigma^2 \log(\beta)}\right) \left(\sqrt{2\sigma^2 \log(\beta)} \!+\! \sigma\right) + \sqrt{m \plog\left(m\sigma^2 \log(\beta) \right)} 
\beta^{-\frac{7}{8}}\right)\,.
}
Therefore by expanding the bracketed product, dropping the non-dominant $1/(\sigma \log(\beta))$ term and upper bounding the sum by the max  
there exists a universal constant $c''' \geq 1$ such that the following implication holds:
\eq{
m\sigma f(\beta) 
\geq c'''\max\set{\sigma,\, \frac{1}{\sqrt{\sigma^2 \log(\beta)}},\, \sqrt{\sigma^2 \log(\beta)},\, \sqrt{mW(m\sigma^2 \log(\beta))} \beta^{-\frac{7}{8}}} 
\!\implies \E[\mu\tau] \geq 0\,.
}
Let $c \geq 1$ be a sufficiently large universal constant and define $\beta$ by
\eq{
\beta_1 &=\frac{m}{c \log_+^{\frac{3}{2}}(m)} &
\beta_2 &=\frac{m\sigma^2}{c \log_+^{\frac{1}{2}}(m\sigma^2)} &
\beta &= 
\begin{cases}
\min\set{\beta_1,\, \beta_2} & \text{if } \min\set{\beta_1,\, \beta_2} \geq 3 \\
1 & \text{otherwise}
\end{cases}
}
If $\beta \geq 3$, then the inequality leading to the implication is shown tediously by applying Lemma \ref{lem:beta} 
to approximate $f(\beta)$ (see \cref{sec:algebra} for the gory algebraic details).
If $\beta = 1$, then $\nu = 0$ and $\E[\mu\tau] \geq 0$ is trivial for the stopping time $\tau = 1$ surely. 
Finally we have the desired lower bound by Lemma \ref{lem:shift},
\eq{
\gamma(\nu', \sigma^2, m) \geq \nu' + \sqrt{2\sigma^2 \log(\beta)} 
\geq \nu' + \sqrt{2\sigma^2 \log\left(\frac{1}{3c}\min\set{\frac{m}{\log_+^{\frac{3}{2}}(m)},\, \frac{m\sigma^2}{\log_+^{\frac{1}{2}}(m\sigma^2)}}\right)}\,.
}
For the upper bound we can proceed naively and exploit the fact that the stopping time must be chosen such that $1 \leq \tau \leq m$ surely.
This time choose $\nu$ such that $\gamma(\nu, \sigma^2, m) = 0$ and let $\beta \geq 1$ be such that $\nu = -\sqrt{2\sigma^2 \log(\beta)}$, which is always
possible by Lemma \ref{lem:positive}.
Then
\eq{
0 
&= \sup_\tau \int^\infty_{-\infty} \theta \E[\tau|\mu = \theta] dP(\theta) 
\leq m\int^\infty_0 \theta dP(\theta) + \int^\nu_{-\infty} \theta dP(\theta) 
=m\sigma f(\beta) - \frac{\sigma}{\sqrt{2\pi}} + \frac{\nu}{2}\,.
}
Rearranging and substituting the definition of $\nu$ and applying Lemma \ref{lem:beta} in the appendix leads to
\eq{
\frac{m}{\sqrt{8\pi} \beta \log \beta} \geq m f(\beta) \geq \frac{1}{\sqrt{2\pi}} + \sqrt{\frac{1}{2} \log \beta} \geq \sqrt{\frac{1}{2} \log \beta}\,.
}
Naive simplification leads to 
$\displaystyle \beta \leq m / \log^{\frac{3}{2}}(m)$
as required.
\end{proofof}

\section{Regret Bounds for the Gittins Index Strategy}\label{sec:finite}

I start with the finite-time guarantees.
Before presenting the algorithm and analysis we need some additional notation.
The empirical mean of arm $i$ after round $t$ is denoted by 
\eq{
\hat \mu_i(t) &= \frac{1}{T_i(t)} \sum_{s=1}^t \ind{I_s = i} X_{s}
& T_i(t) &= \sum_{s \leq t} \ind{I_s = i}\,.
}
I will avoid using $\hat \mu_i(t)$ for rounds $t$ when $T_i(t) = 0$, so this quantity will always be well-defined.
Let $\Delta_{\max} = \max_i \Delta_i$ and
$\Delta_{\min} = \min \set{\Delta_i : \Delta_i > 0}$.
The Gittins index strategy for a flat Gaussian prior is summarised in \cref{alg:flat}.  
Since the flat prior is improper, the Gittins index is not initially defined.
For this reason the algorithm chooses $I_t = t$ in rounds $t \in \set{1,\ldots,d}$ after which the posterior has unit variance 
for all arms and the posterior mean is $\hat \mu_i(d)$. An alternative interpretation of this strategy is the limiting strategy
as the prior variance tends to infinity for all arms. 

\vspace{-0.7cm}
\begin{center}
\begin{minipage}{12cm}
\begin{algorithm}[H]
\KwIn{$d$, $n$} 
Choose each arm once \\
\For{$t \in d+1,\ldots,n$} {
Choose $\displaystyle I_t = \argmax_i \gamma\left(\hat \mu_i(t-1), T_i(t-1)^{-1}, n - t + 1\right)$
}
\caption{Gittins strategy with flat Gaussian prior}\label{alg:flat}
\end{algorithm}
\end{minipage}
\end{center}

\begin{theorem}\label{thm:finite}
Let $\pi$ be the strategy given in \cref{alg:flat}. Then there exist universal constants $c, c' > 0$ such that
\eq{
R^\pi_\mu(n) &\leq c \sum_{i: \Delta_i > 0} \left(\frac{\log(n)}{\Delta_i} + \Delta_i \right) 
&
R^\pi_\mu(n) &\leq c' \left(\sqrt{dn \log(n)} + \sum_{i=1}^d \Delta_i\right)
}
\end{theorem}
While the problem dependent bound is asymptotically optimal up to constant factors, the problem independent bound is sub-optimal by a factor of $\sqrt{\log(n)}$
with both the MOSS algorithm by \cite{AB09} and OCUCB by \cite{Lat15-ucb} matching the lower bound of $\Omega(\sqrt{dn})$ given by \cite{ACFS95}.
The main difficulty in proving Theorem \ref{thm:finite} comes from the fact that the Gittins index is smaller than the upper confidence
bound used by UCB. This is especially true as $t$ approaches the horizon when the Gittins index tends towards the empirical mean, while the UCB index
actually grows. The solution is (a) to use very refined concentration guarantees and (b) to show that the Gittins strategy chooses near-optimal arms
sufficiently often for $t \leq n/2$, ensuring that the empirical mean of these arms is large enough that bad arms are not chosen too often 
in the second half.
Before the proof we require some additional definitions and lemmas.
Assume for the remainder of this section that 
$\mu_1 \geq \mu_2 \geq \ldots \geq \mu_d$, which
is non-restrictive, since if this is not the case, then the arms can simply be re-ordered.
Let $F$ be the event that there exists a $t \in \set{1,\ldots,n}$ and $i \in \set{1,\ldots, d}$ such that
\eq{
\left|\hat \mu_i(t) - \mu_i\right| \geq \sqrt{\frac{2}{T_i(t)} \log (dn^2)}\,.
}
The index of the $i$th arm in round $t$ is abbreviated to 
$\gamma_i(t) = \gamma(\hat \mu_i(t-1), T_i(t-1)^{-1}, n - t + 1)$,
which means that for rounds $t > d$ \cref{alg:flat} is choosing $I_t = \argmax_i \gamma_i(t)$.
Define random variable 
$Z = \mu_1 - \min_{1 \leq t \leq n/2} \gamma_1(t)$,
which measures how far below $\mu_1$ the index of the first arm may fall 
some time in the first $n/2$ rounds.
For each $i \in \set{1,\ldots, d}$ define 
\eqn{
\label{def:u}
u_i = \ceil{\frac{32}{\Delta_i^2} \log(2dn^2)}\,.
}
Now we are ready for the lemmas. First is the key concentration inequality that controls the probability
that the Gittins index of the optimal arm drops far below the true mean.
It is in the proof of \cref{lem:conc} that the odd-looking powers of the logarithmic terms in the bounds on the Gittins index are justified.
Any higher power would lead to super-logarithmic regret, while a lower power could potentially lead to a sub-optimal trade-off between
failure probability and exploration. 

\begin{lemma}\label{lem:conc}
Let $c > 0$ be a universal constant, and
$\Delta > 0$ and $Y_1,Y_2,\ldots$ be a sequence of i.i.d.\ random variables with $Y_t \sim \mathcal N(0, 1)$ and $S_t = \sum_{s=1}^t Y_s$. 
Then there exist universal constants $c' > 0$ and $n_0 \in \N$ such that whenever $n \geq n_0$.
\eq{
&\P{\exists t : S_t \geq t\Delta + \max\set{0,\,\, \sqrt{2t \log\left(\frac{cn}{2\log_{+}^{\frac{3}{2}}(n/2)}\right)}}} 
\leq c' \cdot \frac{\log(n)}{n\Delta^2} \\
&\P{\exists t : S_t \geq t\Delta + \max\set{0,\,\, \sqrt{2t \log\left(\frac{cn}{2t \log_{+}^{\frac{1}{2}}\!\!\left(\frac{n}{2t}\right)}\right)}}} 
\leq c' \cdot \frac{\logp(n\Delta^2)}{n\Delta^2}\,.
}
\end{lemma}
The proof of Lemma \ref{lem:conc} follows by applying a peeling device and may be found in \cref{sec:lem:conc}.
Note that similar results exist in the literature. For example, by \cite{AB09,PR13} and presumably others.
What is unique here is that the improved concentration guarantees for Gaussians are also being exploited.

\begin{assumption}
Assume that $n \geq n_0$, which is non-restrictive since $n_0$ is a universal constant and \cref{thm:finite} holds trivially
with $c = n_0$ for $n \leq n_0$.
\end{assumption}

\begin{lemma}\label{lem:Z}
There exists a universal constant $c > 0$ such that
for all $\Delta > 0$ we have 
\eq{
\P{Z \geq \Delta} \leq c\cdot \frac{\log(n) + \logp(n\Delta^2)}{n\Delta^2}\,.
}
\end{lemma}

\begin{proof}
Apply \cref{thm:gittins} and \cref{lem:conc} and the fact that $m = n - t+1 \geq n/2$ for $t \leq n/2$.
\end{proof}

\begin{lemma}\label{lem:symmetry}
$\displaystyle \E[T_i(n)] \leq n / i$.
\end{lemma}

\begin{proof}
The result follows from the assumption that $\mu_1 \geq\ldots \geq \mu_d$, the definition of the algorithm, the symmetry of the Gaussian density and 
because the exploration bonus due to Gittins index is 
shift invariant (Lemma \ref{lem:shift}).
\end{proof}

\begin{lemma}\label{lem:F}
$\displaystyle \P{F} \leq 1/n$. 
\end{lemma}

\begin{proof}
For fixed $T_i(t) = u$ apply the standard tail inequalities for the Gaussian (Lemma \ref{lem:gaussian}) to bound
$\mathbb{P}\{|\hat \mu_i(t) - \mu_i| \geq \sqrt{2 \log(2dn^2)/u}\} \leq 1/(dn^2)$.
The result is completed by applying the union bound over all arms and values of $u \in \set{1,\ldots,n}$.
\end{proof}

\begin{lemma}\label{lem:half}
If $F$ does not hold and $i$ is an arm such that $Z < \Delta_i / 2$, then $T_i(n/2) \leq u_i$.
\end{lemma}

\begin{proof}
Let $t \leq n/2$ be some round such that $T_i(t - 1) = u_i$. Then by \cref{thm:gittins} 
and the definitions of $F$ and $u_i$ we have
\eq{
\gamma_i(t) 
&\leq \hat \mu_i(t) + \sqrt{\frac{2}{u_i} \log(n)}  
\leq \mu_i + \sqrt{\frac{2}{u_i} \log(2dn^2)} + \sqrt{\frac{2}{u_i} \log(n)} \\
&\leq \mu_i + \frac{\Delta_i}{2} 
= \mu_1 - \frac{\Delta_i}{2} 
< \mu_1 - Z 
\leq \gamma_1(t)\,.
}
Therefore $I_t \neq i$ and so $T_i(n/2) \leq u_i$.
\end{proof}

\begin{proofof}{\cref{thm:finite}}
The regret may be re-written as
\eqn{
\label{eq:decomp}
R^\pi_\mu(n) = \sum_{i=1}^d \Delta_i \E[T_i(n)]\,.
}
To begin, let us naively bound the regret for nearly-optimal arms. Define a set $C \subset \set{1,\ldots, d}$ by
\eq{
C = \set{i : 0 < \Delta_i \leq 4\sqrt{\frac{8i}{n} \log(2dn^2)}}
= \set{i : 0 < \Delta_i \leq \frac{128i}{n\Delta_i} \log(2dn^2)}\,.
}
Therefore by Lemma \ref{lem:symmetry} we have
\eqn{
\label{eq:C}
\sum_{i \in C} \Delta_i \E[T_i(n)] \leq \sum_{i \in C} \frac{128}{\Delta_i} \log(2dn^2)\,.
}
For arms not in $C$ we need to do some extra work. By Lemma \ref{lem:F} 
\eqn{
\sum_{i \notin C} \Delta_i \E[T_i(n)] 
&\leq \P{F} n \Delta_{\max} + \E\left[\ind{\neg F} \sum_{i \notin C} \Delta_i T_i(n)\right] \nonumber \\
&\leq \Delta_{\max} + \E\left[\ind{\neg F} \sum_{i \notin C} \Delta_i T_i(n)\right]\,. \label{eq:de}
}
From now on assume $F$ does not hold while bounding the second term.
By Lemma \ref{lem:half}, if $Z < \Delta_i / 2$, then $T_i(n/2) \leq u_i$.
Define disjoint (random) sets $A, B \subseteq \set{1,\ldots,d}$ by
\eq{
B = \set{i : Z < \Delta_i / 2 \text{ and } \sum_{j \geq i} u_j \leq \frac{n}{4}} \quad \text{and} \quad A = \set{1,\ldots,d} - B\,.
}
The set $A$ is non-empty because $u_1 = \infty$, which implies that $1 \in A$.
Let $i = \max A$, which satisfies either either $Z \geq \Delta_i / 2$ or $n \leq 4\sum_{j \geq i} u_j$. Therefore
\eqn{
\sum_{k \in A} \Delta_k T_k(n)
&\leq n \Delta_i 
\leq \max\set{2nZ \ind{\Delta_{\min}/2 \leq Z}, 4\Delta_i \sum_{j \geq i} u_j} \nonumber \\
&\leq 2nZ \ind{\Delta_{\min}/2 \leq Z} + 4\sum_{j : \Delta_j > 0} \Delta_j u_j\,. \label{eq:A}
}
The next step is to show that there exists an arm in $A$ that is both nearly optimal and has been chosen sufficiently
often that its empirical estimate is reasonably accurate. This arm can then be used to show that bad arms are not chosen
too often.
From the definition of $B$ and because $F$ does not hold we have
$\sum_{i \in A} T_i(n/2) 
= n/2 - \sum_{i \in B} T_i(n/2)
\geq n/2 - \sum_{i \in B} u_i 
\geq n/4$. 
Therefore there exists an $i \in A$ such that $T_i(n/2) \geq n/(4|A|)$. 
Suppose $j \notin A \cup C$ and $\Delta_j \geq 4\Delta_i$ and $T_j(t) = u_j$. Then
\eqn{
\gamma_j(t) 
&\leq \mu_j + \frac{\Delta_j}{2} 
= \mu_i - \frac{\Delta_j}{2} + \Delta_i 
\leq \mu_i - \frac{\Delta_j}{4} 
< \mu_i - \sqrt{\frac{8j}{n} \log(dn^2)} \nonumber \\ 
&\leq \mu_i - \sqrt{\frac{8|A|}{n} \log(dn^2)}  
\leq \hat \mu_i(t) 
\leq \gamma_i(t)\,. \label{eq:A2}
}
Therefore $I_t \neq j$ and so $T_j(n) \leq u_j$.
Now we consider arms with $\Delta_j < 4\Delta_i$. By the same reasoning as in \cref{eq:A} we have
\eq{
\sum_{j : \Delta_j < 4\Delta_i} \Delta_j T_j(n) \leq 4n\Delta_i \leq 8nZ \ind{\Delta_{\min}/2 \leq Z} + 16 \sum_{j : \Delta_j > 0} \Delta_j u_j\,.
}
Finally we have done enough to bound the regret due to arms not in $C$. By the above display and \cref{eq:A} and the sentence after \cref{eq:A2} we have 
\eq{
\sum_{i \notin C} \ind{\neg F} \Delta_i T_i(n)
\leq 10n Z \ind{\Delta_{\min}/2 \leq Z} + 20 \sum_{j : \Delta_j > 0} \Delta_j u_j + \sum_{j \notin C} \Delta_j u_j\,.
}
Combining this with \cref{eq:C} and the regret decomposition \cref{eq:decomp} and \cref{eq:de} we have
\eqn{
\label{eq:rnear}
R^\pi_\mu(n) \leq \Delta_{\max} + \sum_{j \in C} \frac{128}{\Delta_j} \log(dn^2) + 21 \sum_{j : \Delta_j > 0} \Delta_j u_j + 
10n\E\left[Z \ind{\Delta_{\min}/2 \leq Z}\right]\,.
}
From Lemma \ref{lem:Z} there exists a universal constant $c'' > 0$ such that
\eq{
\E Z \ind{\Delta_{\min} / 2 \leq Z} 
&\leq \int^\infty_{\Delta_{\min}/2} \P{Z \geq z} dz + \frac{\Delta_{\min}}{2} \P{Z \geq \frac{\Delta_{\min}}{2}} \\
&\leq \frac{c'' (\log(n) + \logp(n\Delta_{\min}^2))}{n\Delta_{\min}}\,.
}
Substituting into \cref{eq:rnear} and inserting the definition of $u_i$ and naively simplifying completes the proof of the problem dependent
regret bound. 
To prove the second bound in \cref{thm:finite} it suffices to note that the total regret due to arms with $\Delta_i \leq \sqrt{d/n \log(n)}$ is
at most $\sqrt{nd \log(n)}$.
\end{proofof}

\vspace{-0.75cm}
\subsubsect{Asymptotic regret for non-flat prior}
The previous results relied heavily on the choice of a flat prior.
For other (Gaussian) priors it is still possible to prove regret guarantees, but now with a dependence on the prior.
For the sake of simplicity I switch to asymptotic analysis and show that any negative effects of a poorly chosen prior
wash out for sufficiently large horizons.
Let $\nu_i$ and $\sigma_i^2$ be the prior mean and variance for arm $i$ and let $\nu_i(t)$ and $\sigma_i^2(t)$ denote the posterior 
mean and variance at the end of round $t$. A simple computation shows that
\eqn{
\nu_i(t) &= \left(\frac{\nu_i}{\sigma^2_i} + \sum_{s=1}^t  \ind{I_s = i} X_s\right) \!\!\!\Bigg/ \!\!\! \left(\frac{1}{\sigma^2_i} + T_i(t)\right) & 
\label{eq:update}
\sigma_i^2(t) &= \left(\frac{1}{\sigma^2_i} + T_i(t)\right)^{-1}\,.
}
\begin{flalign}
\text{The strategy chooses }
\label{eq:nonflat}
I_t = \argmax_i \gamma(\nu_i(t-1), \sigma_i^2(t-1), n - t + 1)\,.
&&
\end{flalign}

\begin{theorem}\label{thm:asymptotic}
Assume that $\sigma_i^2 > 0$ for all $i$.
Let $\pi_n$ be the strategy given in \cref{eq:nonflat}, then there exists a universal $c > 0$ such that for $\mu \in \R^d$, \,
$\limsup_{n\to\infty} R^{\pi_n}_\mu(n) / \log n \leq c \sum_{i : \Delta_i > 0} \Delta_i^{-1}$.
\end{theorem}

The proof may be found in \cref{sec:thm:asymptotic}.
It should be emphasised that the limit in \cref{thm:asymptotic} is taken over a sequence of strategies and an increasing horizon. This is in contrast to similar
results for UCB where the strategy is fixed and only the horizon is increasing.

\section{Discussion}\label{sec:conc}
I have shown that the Gittins strategy enjoys finite-time regret guarantees, which explicates
the excellent practical performance. The proof relies on developing tight bounds on the index, which
are asymptotically exact. If the prior is mis-specified, then the resulting Gittins index strategy is order-optimal asymptotically, but its finite-time
regret may be significantly worse.
Experimental results show the Gittins strategy with a flat improper Gaussian prior is never much worse and often much better than the best frequentist algorithms (please see \cref{sec:exp}). 
The index is non-trivial to compute, but reasonable accuracy is possible for horizons of $O(10^4)$. For larger horizons I propose a new index inspired
by the theoretical results that is efficient and competitive with the state-of-the-art, at least in the worst case regime for which experiments were performed.
There are a variety of open problems some of which are described below.

\subsubsect{Alternative prior and noise models}
The Gaussian noise/prior was chosen for its simplicity, but
bounded or Bernoulli rewards are also interesting. In the latter, the Gittins index
can be computed by using a Beta prior and dynamic programming. Asymptotic approximations by \cite{BK97} suggest that one might expect 
to derive the KL-UCB algorithm in this manner \citep{CGMMS13},
but finite-time bounds would be required to obtain regret guarantees. Many of the concentration inequalities used in the Gaussian
case have information-theoretic analogues, so in principle I expect that substituting these results into the current technique should lead to good results
\citep{Gar13}.

\subsubsect{Finite-time impact of mis-specified prior}
For non-flat priors I only gave asymptotic bounds on the regret, showing that the Gittins index strategy will eventually recover from even the most poorly mis-specified
prior.
Of course it would be nice to fully understand how long the algorithms takes to recover by analysing its finite-time regret.
Some preliminary work has been done on this question in a simplified setting for Thompson sampling by \cite{LL15}.
Unfortunately the results will necessarily be quite negative. If a non-flat prior is chosen in such a way that
the resulting algorithm achieves unusually small regret with respect to a particular arm, then the regret it incurs on the remaining arms
must be significantly larger \citep{Lat15-unfair}. In short, there is a large price to pay for favouring one arm over another and the Bayesian
algorithm cannot save you. This is in contrast to predictive settings where a poorly chosen prior is quickly washed away by data. 

\subsubsect{Asymptotic optimality}
\cref{thm:asymptotic} shows that the Gittins index strategy is eventually \textit{order-optimal} for any choice of prior, but the leading constant 
does not match optimal rate. The reason is that the failure probability for which the algorithm suffers linear regret appears to be $O(\log(n)/n)$ and
so this term is not asymptotically insignificant. In contrast, the UCB confidence level is chosen such that the failure probability is $O(1/n)$, which
is insignificant for large $n$.

\subsubsect{Horizon effects}
Recent bandit algorithms including MOSS \citep{AB09} and OCUCB \citep{Lat15-ucb} exploit the knowledge of the horizon. The Gittins strategy also 
exploits this knowledge, and it is not clear
how it could be defined without a horizon (the index tends to infinity as $n$ increases). The 
most natural approach would be to choose a prior on the unknown horizon, but what
prior should you choose and is there hope to compute the index in that case?\footnote{An exponential prior leads to the discounted 
setting for which anytime regret guarantees are unlikely to exist, but where computation is efficient. A power-law would be a more natural
choice, but the analysis becomes very non-trivial in that case.}
You can also ask what is the benefit of knowing the horizon? The exploration bonus of the Gittins
strategy (and MOSS and OCUCB) tend to zero as the horizon approaches, which makes these algorithms more 
aggressive than anytime algorithms such as UCB and Thompson sampling and improves practical performance.

\subsubsect{Extending the model and computation issues}
For finite-armed bandits and simple priors the index may be approximated in polynomial time, with practical computation possible to horizons of $\sim\!\!\! 10^4$.
It is natural to ask whether or not the computation techniques can be improved to compute the index for larger horizons. 
I am also curious to know if the classical results can be extended to more complicated settings such as linear bandits or partial monitoring.
This has already been done to some extent (eg., restless bandits), but there are many open problems. The recent book by \cite{GGW11} is a broad reference for existing
extensions.

\subsection*{Acknowledgements}

My thanks to Marcus Hutter for several useful suggestions and to Dan Russo for pointing out that the finite-horizon Gittins strategy is
not generally Bayesian optimal.
The experimental component of this research was enabled by the support provided 
by WestGrid (\url{www.westgrid.ca}) and Compute Canada (\url{www.computecanada.ca}).

\appendix
\bibliography{all}
\section{Comparison to Bayes}\label{sec:bayes}
As remarked in the introduction, 
it turns out that both geometric discounting and an infinite horizon are crucial for the interchange argument used in all proofs of the Gittins
index theorem and indeed the result is known to be false in general for non-geometric discounting as shown by \cite{BF85}.
Thus the observation that the Gittins index strategy is not Bayesian optimal in the setting considered here should not come as a surprise, but
is included for completeness.
Let $n = 2$ and $\nu_1 = 0$ and $\sigma^2_1 = 1$ and $\sigma^2_2 = 1/2$. Then the Gittins indices are
\eq{
\gamma(\nu_1, \sigma^2_1, 2) &\approx 0.195183 &
\gamma(\nu_2, \sigma^2_2, 2) &\approx \nu_2 + 0.112689\,.
}
Therefore the strategy based on Gittins index will choose the second action if
$\nu_2 \gtrsim 0.082494$.
Computing the Bayesian value of each choice is possible analytically 
\eq{
\sup_\pi \E\left[\sum_{t=1}^n X_t \Bigg| I_1 = 1\right]
&= \int^\infty_{-\infty} \max\set{\nu_2, \delta} \frac{1}{\sqrt{\pi}} \exp\left(-\delta^2 \right) d\delta \\
&= \frac{\exp\left(-\nu_2^2\right) }{2\sqrt{\pi}} + \frac{\nu_2 + \nu_2 \erf(\nu_2)}{2}\,. \\
\sup_\pi \E\left[\sum_{t=1}^n X_t \Bigg| I_1 = 2\right]
&= \nu_2 + \int^\infty_{-\infty} \max\set{0, \nu_2 + \delta} \frac{1}{\sqrt{\pi/3}} \exp\left(-3\delta^2\right) d\delta \\
&= \nu_2 + \frac{\exp\left(-3\nu_2^2\right)}{\sqrt{2\pi}} + \frac{\nu_2 + \nu_2 \erf\left(\sqrt{3} \nu_2\right)}{2}\,.
}
Solving leads to the conclusion that $I_2 = 2$ is optimal only if $\nu_2 \gtrsim 0.116462$ and hence the Gittins strategy is not Bayesian optimal (it does
not minimise \cref{eq:bayes}).
Despite this, the following result shows that the regret analysis for the Gittins index strategy can also be applied to the 
intractable fully Bayesian algorithm when the number of arms is $d = 2$. The idea is to show that the Bayesian algorithm will never
choose an arm for which a UCB-like upper confidence bound is smaller than the largest Gittins index.
Then the analysis in Section \ref{sec:finite} may be repeated to show that 
that Theorems \ref{thm:finite} and \ref{thm:asymptotic} also hold for the Bayesian algorithm when $d = 2$.

\begin{theorem}\label{thm:bayes-arm}
Assume $d = 2$ and the horizon is $n$.
Let $Q$ be the multivariate Gaussian prior measure with mean $\nu \in \R^d$ and covariance matrix $\Sigma = \diag(\sigma^2)$.
If the Bayesian optimal action is $I_1 = 1$, then
\eq{
\nu_1 + \sqrt{2c\sigma^2_1 \log n} \geq \gamma(\nu_2, \sigma^2_2, n)\,,
}
where $c > 0$ is a universal constant.
\end{theorem}

The proof of \cref{thm:bayes-arm} is surprisingly tricky and may be found in \cref{sec:thm:bayes-arm}.
Empirically the behaviour of the Bayesian algorithm and the Gittins index strategy is 
almost indistinguishable, at least for two arms and small horizons (see \cref{sec:exp}).

\section{Computing the Gittins Index}\label{sec:compute}
\newcommand{\V}{\mathcal V}

I briefly describe a method of computing the Gittins index in the Gaussian case.
A variety of authors have proposed sophisticated methods for computing the Gittins index, both in the discounted
and finite horizon settings \citep[and references therein]{Nin11,CM13}. The noise model used here is Gaussian and so continuous, which seems to make 
prior work inapplicable. Fortunately the Gaussian model is rather special, mostly due to the shift invariance (Lemma \ref{lem:shift}), which can be exploited to
compute and store the index quite efficiently.

Let $\nu \in \R$ and $\sigma^2 > 0$ be the prior mean and variance respectively and $m$ be the number of rounds remaining.
For $t \in \set{1,\ldots,m}$ define independent random variables
\eq{
\eta_t \sim \mathcal N\left(0, \frac{\sigma^2}{1 + (t-1) \sigma^2} \cdot \frac{\sigma^2}{1 + t\sigma^2}\right)\,.
}
Let $\nu_1 = \nu$ and $\nu_t = \nu_{t-1} + \eta_{t-1}$ for $t > 1$. Then 
\eqn{
\label{eq:git-alt}
\gamma(\nu, \sigma^2, m) 
&= \max\set{\gamma : \sup_{1\leq \tau\leq m} \E\left[\sum_{t=1}^\tau (\nu_t - \gamma)\right] = 0} 
}
where the stopping time is with respect to the filtration generated by $\eta_1,\ldots,\eta_m$.
\cref{eq:git-alt} is the Bayesian view of \cref{eq:gittins} with the integral over $\mu$ in that equation incorporated
into the posterior means.
As in the proof of \cref{thm:gittins}, let $\nu$ be such that $\gamma(\nu, \sigma^2, m) = 0$. 
Then
\eq{
0 = \gamma(\nu, \sigma^2, m) = \sup_{1 \leq \tau \leq m} \E\left[\sum_{t=1}^\tau \nu_t\right]\,.
}
The optimal stopping problem above can be solved by finding the root of the following Bellman equation for $t = m$.
\eqn{
\label{eq:bellman}
\V_t(x, \sigma^2) = 
\begin{cases}
x + \E_{\eta \sim \mathcal N(0, \sigma^4/(1 + \sigma^2))} \left[\max\set{0, \V_{t-1}\left(x + \eta, \frac{\sigma^2}{1+\sigma^2}\right)}\right] 
& \text{if } t \geq 1 \\
0 & \text{otherwise}\,.
\end{cases}
}
Then by Lemma \ref{lem:shift} the Gittins index satisfies 
\eqn{
\label{eq:root}
\gamma(0, \sigma^2, m) = \gamma(\nu, \sigma^2, m) - \nu = -\nu =  -\max\set{x : \V_m(x, \sigma^2) = 0}\,.
}
Now \cref{eq:bellman} is a Bellman equation and
conveniently there is a whole field devoted to solving such equations
\citep[and references therein]{BT95}. 
An efficient algorithm that computes the Gittins index to arbitrary precision by solving \cref{eq:bellman} 
using backwards induction and 
approximating $\max\set{0, \V_t(x, \sigma^2)}$ using quadratic splines may be found at 
\burl{https://github.com/tor/libbandit}. 
The motivation for choosing the quadratic is because
the expectation in \cref{eq:bellman} can be computed explicitly in terms of the error function and because $\V_t(x, \sigma^2)$ is
convex in $x$, so the quadratic leads to a relatively good fit with only a few splines.
The convexity of $\V$ also means that \cref{eq:root} can be computed efficiently from a sufficiently good approximation of $\V_m$.
Finally, in order to implement \cref{alg:flat} up to a horizon of $n$ we may need $\gamma(\nu, 1/T, m) = \nu + \gamma(0,1/T, m)$ 
for all $T$ and $m$ satisfying $T + m \leq n$. This can computed by solving $n$ copies of \cref{eq:bellman} with
$\sigma^2 = 1$ and $m \in \set{1,\ldots,n-1}$.
The total running time is $O(n^2 N)$ where $N$ is the maximum number of splines required for sufficient accuracy.
The results can be stored in a lookup table of size $O(n^2)$, which makes the actual simulation of \cref{alg:flat} extremely fast.
Computation time for $n = 10^4$ was approximately 17 hours using 8 cores of a Core-i7 machine.

\section{Experiments}\label{sec:exp}
\pgfplotstableread[comment chars={\%}]{data/exp1.txt}{\tableWorstTwo}
\pgfplotstableread[comment chars={\%}]{data/exp2.txt}{\tableWorstFive}
\pgfplotstableread[comment chars={\%}]{data/exp3.txt}{\tableWorstTen}
\pgfplotstableread[comment chars={\%}]{data/exp4.txt}{\tableApprox}
\pgfplotstableread[comment chars={\%}]{data/exp5.txt}{\tableApproxLong}
\pgfplotstableread[comment chars={\%}]{data/exp6.txt}{\tableWorstLongTwo}
\pgfplotstableread[comment chars={\%}]{data/exp7.txt}{\tableWorstLongFive}
\pgfplotstableread[comment chars={\%}]{data/exp8.txt}{\tableWorstLongTen}
\pgfplotstableread[comment chars={\%}]{data/exp9.txt}{\tableFullBayes}

\pgfplotstableread[comment chars={\%}]{data/index}{\tableIndex}
\pgfplotstableread[comment chars={\%}]{data/index2}{\tableIndexTwo}

I compare the Gittins strategy given in \cref{alg:flat} with UCB, OCUCB \citep{Lat15-ucb} and Thompson sampling (TS) with a flat Gaussian prior \citep{Tho33,AG12}.
Due to the computational difficulties in calculating the Gittins index we are limited to modest horizons.
For longer horizons the Gittins index may be approximation by the lower bound in \cref{thm:gittins}.
I also compare the Gittins index strategy to the Bayesian optimal strategy that can be computed with reasonable accuracy 
for the two-armed case and a horizon of $n = 2000$.
Error bars are omitted from all plots because they are too small to see.
All code is available from \burl{https://github.com/tor/libbandit}. 
\begin{table}[H]
\centering
\begin{tabular}{ll}
\toprule
{\textbf {Algorithm}}                  & {\textbf{Index}}
\renewcommand{\arraystretch}{1.6}
                  \\ \midrule
UCB               & $\hat \mu_i(t-1) + \sqrt{\frac{2}{T_i(t-1)} \log t}$                        \\[0.4cm]
OCUCB             & $\hat \mu_i(t-1) + \sqrt{\frac{3}{T_i(t-1)} \log \left(\frac{2 n}{t}\right)}$      \\[0.4cm]
TS & $\sim \mathcal N(\hat \mu_i(t-1), (T_i(t-1)+1)^{-1})$ \\
\bottomrule
\end{tabular}
\caption{Comparison Algorithms}
\end{table}

\subsubsect{Worst-case regret}
I start with the worst case regret for horizons $n \in \set{10^3, 10^4}$. In all experiments the first arm is optimal and has mean $\mu_1 = 0$.
The remaining arms have $\mu_i = -\Delta$.
In \cref{fig:worst} I plot the expected regret of various algorithms with $d \in \set{2,5,10}$ and varying $\Delta$. 
The results demonstrate that the Gittins strategy significantly outperforms 
both UCB and Thompson sampling, and is a modest improvement on OCUCB in most cases.

\pgfplotsset{cycle list={{red,dotted}, {green!50!black,dashdotdotted}, {blue,dashed}, {black}}}
\pgfplotsset{every axis plot/.append style={line width=1.5pt}}

\newenvironment{customlegend}[1][]{%
  \begingroup
      \csname pgfplots@init@cleared@structures\endcsname
       \pgfplotsset{#1}%
}{%
    \csname pgfplots@createlegend\endcsname
    \endgroup
}%

  \def\addlegendimage{\csname pgfplots@addlegendimage\endcsname}

\newcommand{\defaultaxis}{
        xlabel shift=-5pt,
        ylabel shift=-2pt,
        width=5cm,
        height=4cm,
        legend cell align=left,
        compat=newest}

\begin{figure}[H]
\centering
\begin{tikzpicture}[font=\scriptsize]
\begin{customlegend}[
  legend entries={UCB,Thompson Sampling,OCUCB,Gittins (Algorithm 1)},
  legend columns=-1,
  legend style={column sep=2ex,}]

\addlegendimage{red,dotted,sharp plot}
\addlegendimage{green!50!black,dashdotdotted,sharp plot}
\addlegendimage{blue,dashed,sharp plot}
\addlegendimage{black,sharp plot}
\end{customlegend}

\end{tikzpicture}

  \begin{tikzpicture}[font=\scriptsize]
    \begin{axis}[\defaultaxis,
        xmin=0,
        ymin=0,
        xmax=2,
        ytick={0,10,20,25,30},
        yticklabels={0,10,20,\phantom{300},30},
        xtick={0,2},
        xlabel={\begin{minipage}{3cm}\centering $\Delta$ \\ $n = 10^3$, $d = 2$\end{minipage}},
        ylabel={Expected Regret}]

      \addplot+[] table[x index=0,y index=1] \tableWorstTwo;
      \addplot+[] table[x index=0,y index=3] \tableWorstTwo;
      \addplot+[] table[x index=0,y index=2] \tableWorstTwo;
      \addplot+[] table[x index=0,y index=4] \tableWorstTwo;
    \end{axis}
  \end{tikzpicture}
  \begin{tikzpicture}[font=\scriptsize]
    \begin{axis}[\defaultaxis,
        xmin=0,
        ymin=0,
        xmax=2,
        xtick={0,2},
        ytick={0,20,40,60,65,80},
        yticklabels={0,20,40,60,\phantom{300},80},
        xlabel={\begin{minipage}{3cm}\centering $\Delta$ \\ $n = 10^3$, $d = 5$\end{minipage}},]

      \addplot+[] table[x index=0,y index=1] \tableWorstFive;
      \addplot+[] table[x index=0,y index=3] \tableWorstFive;
      \addplot+[] table[x index=0,y index=2] \tableWorstFive;
      \addplot+[] table[x index=0,y index=4] \tableWorstFive;

    \end{axis}
  \end{tikzpicture}
  \begin{tikzpicture}[font=\scriptsize]
    \begin{axis}[\defaultaxis,
        xmin=0,
        ymin=0,
        xmax=2,
        xtick={0,2},
        xlabel={\begin{minipage}{3cm}\centering $\Delta$ \\ $n = 10^3$, $d = 10$\end{minipage}},]

      \addplot+[] table[x index=0,y index=1] \tableWorstTen;
      \addplot+[] table[x index=0,y index=3] \tableWorstTen;
      \addplot+[] table[x index=0,y index=2] \tableWorstTen;
      \addplot+[] table[x index=0,y index=4] \tableWorstTen;
    \end{axis}
  \end{tikzpicture}

  \noindent
  \begin{tikzpicture}[font=\scriptsize]
    \begin{axis}[\defaultaxis,
        xmin=0,
        ymin=0,
        xmax=0.5,
        xtick={0,0.5},
        xticklabels={0,$\frac{1}{2}$},
        ylabel={Expected Regret},
        xlabel={\begin{minipage}{3cm}\centering $\Delta$ \\ $n = 10^4$, $d = 2$\end{minipage}},
        ]

      \addplot+[] table[x index=0,y index=1] \tableWorstLongTwo;
      \addplot+[] table[x index=0,y index=3] \tableWorstLongTwo;
      \addplot+[] table[x index=0,y index=2] \tableWorstLongTwo;
      \addplot+[] table[x index=0,y index=4] \tableWorstLongTwo;
    \end{axis}
  \end{tikzpicture}
  \begin{tikzpicture}[font=\scriptsize]
    \begin{axis}[\defaultaxis,
        xmin=0,
        ymin=0,
        xmax=0.5,
        xtick={0,0.5},
        xticklabels={0,$\frac{1}{2}$},
        xlabel={\begin{minipage}{3cm}\centering $\Delta$ \\ $n = 10^4$, $d = 5$\end{minipage}},
        ]
      \addplot+[] table[x index=0,y index=1] \tableWorstLongFive;
      \addplot+[] table[x index=0,y index=3] \tableWorstLongFive;
      \addplot+[] table[x index=0,y index=2] \tableWorstLongFive;
      \addplot+[] table[x index=0,y index=4] \tableWorstLongFive;
    \end{axis}
  \end{tikzpicture}
  \begin{tikzpicture}[font=\scriptsize]
    \begin{axis}[\defaultaxis,
        xmin=0,
        ymin=0,
        xmax=0.5,
        xtick={0,0.5},
        xticklabels={0,$\frac{1}{2}$},
        xlabel={\begin{minipage}{3cm}\centering $\Delta$ \\ $n = 10^4$, $d = 10$\end{minipage}}]

      \addplot+[] table[x index=0,y index=1] \tableWorstLongTen;
      \addplot+[] table[x index=0,y index=3] \tableWorstLongTen;
      \addplot+[] table[x index=0,y index=2] \tableWorstLongTen;
      \addplot+[] table[x index=0,y index=4] \tableWorstLongTen;
    \end{axis}
  \end{tikzpicture}

  \caption{Worst case regret comparison}\label{fig:worst}
\end{figure}
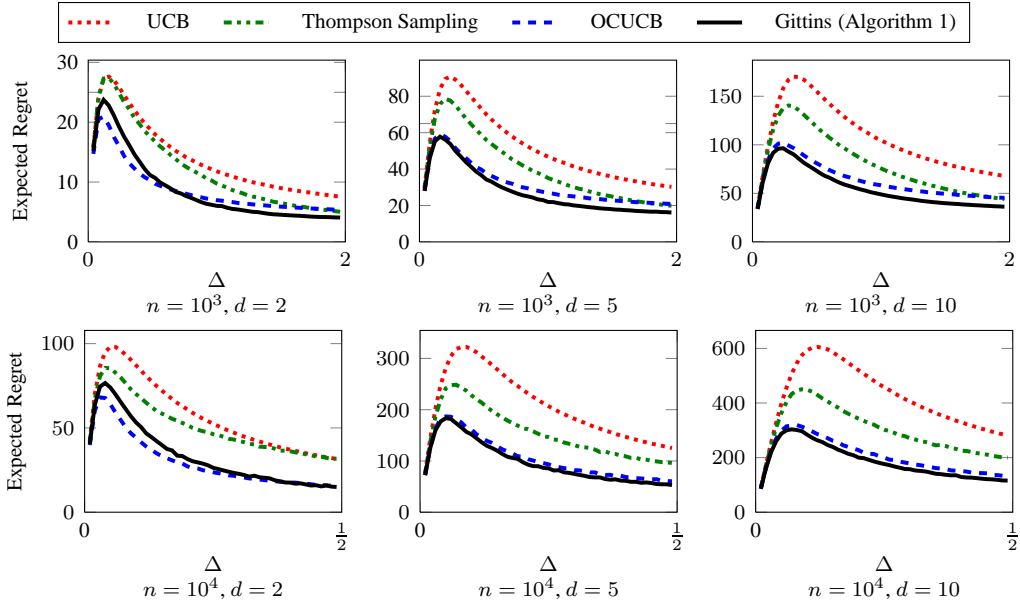

\subsubsect{Approximating the Gittins index}
Current methods for computing the Gittins index are only practical for horizons up to $O(10^4)$.
For longer horizons it seems worthwhile to find a closed form approximation.
Taking inspiration from \cref{thm:gittins}, define $\tilde \gamma(\nu, \sigma^2, m) \approx \gamma(\nu, \sigma^2, m)$ by
\eq{
\tilde \gamma(\nu, \sigma^2, m) &= \nu + \sqrt{2\sigma^2 \log \beta(\sigma^2, m)}
\qquad \text{where} \qquad \\
\beta(\sigma^2, m) &= \max\set{1,\,\, \frac{1}{4} \min\set{\frac{m}{\log^{\frac{3}{2}}(m)},\,\, \frac{m\sigma^2}{\log^{\frac{1}{2}}(m\sigma^2)}}}\,.
}
This is exactly the lower bound in \cref{thm:gittins}, but with the leading constant chosen (empirically) to be $1/4$.
I compare the indices in two key regimes.
First fix $\nu = 0$ and $\sigma^2 = 1$ and vary the horizon.
In the second regime the horizon is fixed to $m = 10^3$ and $\sigma^2 = 1/T$ is varied. The results (\cref{fig:approx-index})
suggest a relatively good fit.

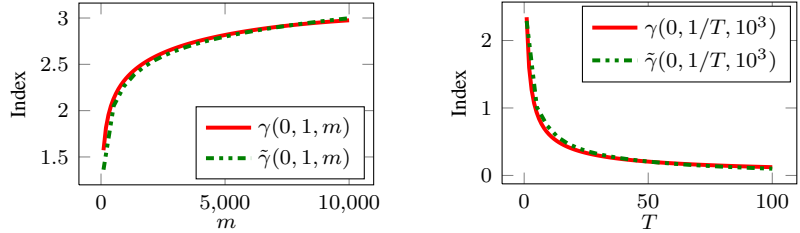
\begin{figure}[H]
\centering
\begin{tikzpicture}[font=\scriptsize]
\begin{axis}[\defaultaxis,
       xlabel={$m$},
       width=5.5cm,
       legend pos=south east,
       scaled x ticks=false,
       xticklabel style={/pgf/number format/fixed},
       ylabel=Index]

\addplot+[solid] table[x index=0,y index=2] \tableIndex;
\addlegendentry{$\gamma(0, 1, m)$};
\addplot+[domain=100:10000] {sqrt(2 * min(ln(x / 4 / ln(x)^(1/2)), ln(x / 4 / ln(x)^(3/2))))};
\addlegendentry{$\tilde \gamma(0, 1, m)$}
\end{axis}
\end{tikzpicture}
\hspace{0.5cm}
\begin{tikzpicture}[font=\scriptsize]
\begin{axis}[\defaultaxis,
       xlabel={$T$},
       width=5.5cm,
       ylabel=Index]

\addplot+[solid] table[x index=1,y index=2] \tableIndexTwo;
\addlegendentry{$\gamma(0, 1/T, 10^3)$};
\addplot+[domain=1:100] {sqrt(2 / x * min(ln(1000 / 4 / x / ln(1000/x)^(1/2)), ln(1000 / 4 / ln(1000)^(3/2))))};
\addlegendentry{$\tilde \gamma(0, 1/T, 10^3)$};
\end{axis}
\end{tikzpicture}

\caption{Approximation of the index}\label{fig:approx-index}
\end{figure}

The approximated index strategy is reasonably competitive with the Gittins strategy (\cref{fig:approx}).
For longer horizons the Gittins index cannot be computed in reasonable time, but the approximation can be compared to other efficient
algorithms such as OCUCB, Thompson sampling and UCB (\cref{fig:approx}). The approximated version Gittins index 
is performing well compared to Thompson sampling and UCB, and is only marginally worse than OCUCB. As an added bonus, by cloning the
proofs of Theorems \ref{thm:finite} and \ref{thm:asymptotic} it can be shown that the approximate index enjoys the same regret guarantees
as the Gittins strategy/UCB.

\begin{figure}[H]
  \begin{tikzpicture}[font=\scriptsize]
    \begin{axis}[\defaultaxis,
        xmin=0,
        ymin=0,
        xmax=2,
        width=7cm,
        xlabel={\begin{minipage}{5cm}\centering $\Delta$ \\ $n = 10^3$, $d = 5$, $\mu_1 = 0$, $\mu_{\geq } = -\Delta$\end{minipage}},
        ylabel={Expected Regret}]
      \addplot+[black,solid] table[x index=0,y index=1] \tableApprox;
      \addlegendentry{Gittins};

      \addplot+[black,dashed] table[x index=0,y index=2] \tableApprox;
      \addlegendentry{Gittins Approximation};
    \end{axis}
  \end{tikzpicture}
  \begin{tikzpicture}[font=\scriptsize]
    \begin{axis}[\defaultaxis,
        xmin=0,
        ymin=0,
        xmax=0.2,
        width=8cm,
        xtick={0,0.05,0.1,0.15,0.2},
        xlabel={$\Delta$},
        xlabel={\begin{minipage}{5cm}\centering $\Delta$ \\ $n = 5\times 10^4$, $d = 5$, $\mu_1 = 0$, $\mu_{\geq } = -\Delta$\end{minipage}},
        xticklabel style={/pgf/number format/fixed},
        ylabel={Expected Regret}]
      \addplot+[] table[x index=0,y index=2] \tableApproxLong;
      \addlegendentry{UCB};

      \addplot+[] table[x index=0,y index=4] \tableApproxLong;
      \addlegendentry{Thompson Sampling};

      \addplot+[] table[x index=0,y index=3] \tableApproxLong;
      \addlegendentry{OCUCB};

      \addplot+[] table[x index=0,y index=1] \tableApproxLong;
      \addlegendentry{Gittins Approximation};

    \end{axis}
  \end{tikzpicture}
  \caption{Regret for approximate Gittins index strategy}\label{fig:approx}
\end{figure}
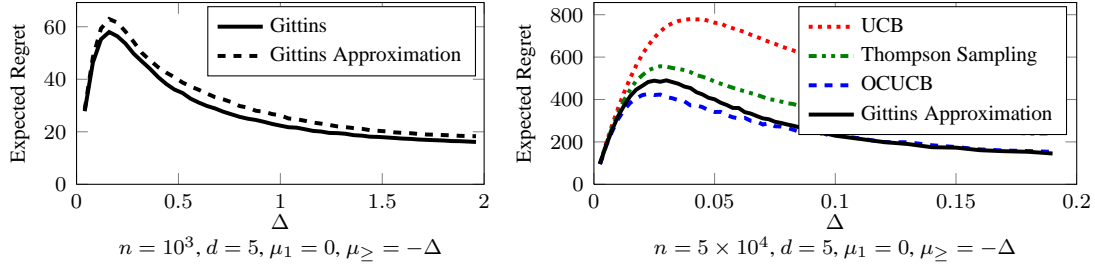

\subsubsect{Comparing Gittins and Bayes}
Unfortunately the Bayesian optimal strategy is not typically an index strategy. The two natural exceptions occur when the rewards
are discounted geometrically or when there are only two arms and the return of the second arm is known. Despite this, the
Bayesian optimal strategy can be computed for small horizons and a few arms using a similar method as the Gittins index.
Below I compare the Gittins index against the Bayesian optimal and OCUCB for $n = 2000$ and $d = 2$ and $\mu_1 = 0$ and $\mu_2 = -\Delta$.
The results show the similar behaviour of the Gittins index strategy and the Bayesian optimal strategy, while OCUCB is making a different
trade-off (better minimax regret at the cost of worse long-term regret).

\begin{figure}[H]
  \centering
  \begin{tikzpicture}[font=\scriptsize]
    \begin{axis}[\defaultaxis,
        xmin=0,
        ymin=0,
        xmax=2,
        width=8cm,
        xtick={0,2},
        xticklabels={0,2},
        xlabel={\begin{minipage}{3cm}\centering $\Delta$ \\ $n = 2000$, $d = 2$\end{minipage}}]

      \addplot+[blue,dashed] table[x index=0,y index=1] \tableFullBayes;
      \addlegendentry{OCUCB};
      \addplot+[black,solid] table[x index=0,y index=2] \tableFullBayes;
      \addlegendentry{Gittins};
      \addplot+[red,dashed] table[x index=0,y index=3] \tableFullBayes;
      \addlegendentry{Bayes};
    \end{axis}
  \end{tikzpicture}
  \caption{Comparing Gittins and Bayesian optimal}\label{fig:bayes}
\end{figure}
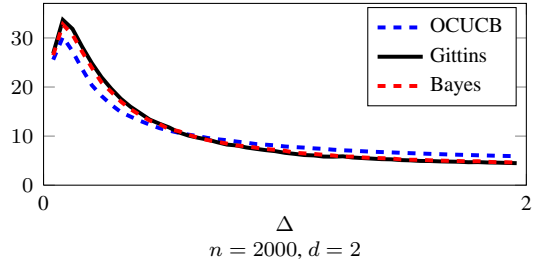

\section{Proof of Lemma \ref{lem:bandit}}\label{sec:lem:bandit}

For parts (1) and (2) we have $\theta \in (-\infty, 0)$. Define
$t_\theta \in \set{1,2,\ldots}$ by
\eq{
t_\theta 
&= \min\set{t \geq \frac{1}{\nu^2} : \sqrt{\frac{4}{t} \log (4t\nu^2)} \leq -\theta/2}\,. 
}
Then
\eq{
\E[\tau|\mu= \theta] 
&= \sum_{t=1}^m \P{\tau \geq t| \mu = \theta} 
\leq t_\theta + \sum_{t=t_\theta+1}^m \P{\tau \geq t | \mu = \theta} \\
&\leq t_\theta + \sum_{t=t_\theta}^{m-1} \P{\hat \mu_t + \sqrt{\frac{4}{t} \log\left(4t\nu^2\right)} > 0\Bigg| \mu = \theta} \\
&\leq t_\theta + \sum_{t=t_\theta}^{m-1} \P{\hat \mu_t \geq \theta / 2 | \mu= \theta}  
\leq t_\theta + \sum_{t=t_\theta}^\infty \exp\left(-\frac{t \theta^2}{8}\right) 
\leq t_\theta + \frac{8}{\theta^2}\,.
}
The results follow by naively bounding $t_\theta$.
For part 3, assume $\mu \geq 0$ and define $S_t = t \hat \mu_t$ and $t_k = \ceil{1/\nu^2} 2^k$. By using the peeling device we obtain
\eq{
\P{\tau < m|\mu = \theta} 
&= \P{\exists t \geq \frac{1}{\nu^2} : \hat \mu_t + \sqrt{\frac{4}{t} \log(4t\nu^2)} \leq 0 \Bigg|\mu = \theta} \\
&\leq \sum_{k=0}^\infty \P{\exists t \leq t_{k+1} : S_t + \sqrt{2 \cdot t_{k+1} \log (4t_k \nu^2)} \leq 0 \Big|\mu = \theta} \\
&\leq \sum_{k=0}^\infty \frac{1}{4t_k \nu^2} 
\leq \sum_{k=0}^\infty \frac{1}{4 \cdot 2^k} 
= \frac{1}{2}\,.
}
Therefore $\E[\tau|\mu = \theta] \geq m/2$.

\section{Concentration of Gaussian Random Variables}\label{sec:gaussian}

These lemmas are not new, but are collected here with proofs for the sake of completeness.

\begin{lemma}\label{lem:gaussian}
Let $\sigma \geq 0$ and $\nu < 0$ and $X \sim \mathcal N(\nu, \sigma^2)$, then for $x \geq \nu$ 
\begin{enumerate}
\item $\displaystyle \P{X \geq x} \leq \exp\left(-\frac{(\nu - x)^2}{2\sigma^2}\right)$.
\item $\displaystyle \E[X \ind{X \leq 0}] \geq \nu - \sigma / \sqrt{2\pi}$.
\end{enumerate}
\end{lemma}

\begin{proof}
The first is well known. The second follows since
\eq{
\E[X \ind{X \leq 0}] 
&= \int^0_{-\infty} \frac{x}{\sqrt{2\pi\sigma^2}} \exp\left(-\frac{(\nu - x)^2}{2\sigma^2}\right) dx \\
&= \frac{\nu}{2} \erfc\left(\frac{\nu}{\sqrt{2\sigma^2}}\right) - \sqrt{\frac{\sigma^2}{2\pi}}\exp\left(-\frac{\nu^2}{2\sigma^2}\right) \\
&\geq \nu - \sigma / \sqrt{2\pi} 
} 
as required.
\end{proof}

\begin{lemma}\label{lem:maximal}
Let $\delta \in (0,1)$ and $\Delta \geq 0$ and $X_1,\ldots,X_n$ be a sequence of i.i.d.\ random variables with $X_t \sim \mathcal N(0, 1)$ and $S_t = \sum_{s=1}^t X_s$. Then
\eq{
\P{\exists t \leq n : S_t \geq n\Delta + \sqrt{2 n \log\left(\frac{1}{\delta}\right)}} 
\leq \frac{\delta}{\sqrt{\pi \log(1/\delta)}} \exp\left(-\frac{n\Delta^2}{2}\right)\,.
}
\end{lemma}

\begin{proof}
Let $\epsilon > 0$. By the reflection principle we have $\P{\exists t \leq n : S_t \geq \epsilon} \leq 2 \P{S_n \geq \epsilon}$ and
\eq{
2\P{S_n \geq \epsilon} 
&= 2\int^\infty_\epsilon \frac{1}{\sqrt{2\pi n}} \exp\left(-\frac{x^2}{2n}\right) dx \\
&\leq \frac{2}{\epsilon} \int^\infty_\epsilon \frac{x}{\sqrt{2\pi n}} \exp\left(-\frac{x^2}{2n}\right) dx \\
&= \frac{1}{\epsilon} \sqrt{\frac{2n}{\pi}} \exp\left(-\frac{\epsilon^2}{2n}\right)\,.
}
There result is completed by substituting $\epsilon = n\Delta + \sqrt{2n \log(1/\delta)}$ and naive simplification.
\end{proof}

\section{Proof of Lemma \ref{lem:conc}}\label{sec:lem:conc}

For the first part we can apply Lemma \ref{lem:maximal} and the union bound.
\eq{
&\P{\exists t \leq n : S_t \geq t\Delta + \sqrt{2t \log\left(\frac{cn}{2\log_+^{\frac{3}{2}}(n/2)}\right)}} \\ 
&\qquad\qquad\leq \sum_{t=1}^{n} \P{S_t \geq t\Delta + \sqrt{2t \log\left(\frac{cn}{2\log_+^{\frac{3}{2}}(n/2)}\right)}} \\
&\qquad\qquad\leq \frac{2\log_+^{\frac{3}{2}}(n/2)}{cn \sqrt{\pi \log\left(\frac{cn}{2\log_+^{\frac{3}{2}}(n/2)}\right)}} \sum_{t=1}^\infty \exp\left(-\frac{t\Delta^2}{2}\right) \\
&\qquad\qquad\leq \frac{\log(n)}{n\Delta^2} \cdot \frac{4\log_+^{\frac{3}{2}}(n/2)}{c \log(n) \sqrt{\pi \log\left(\frac{cn}{2\log_+^{\frac{3}{2}}(n/2)}\right)}}\,. 
}
The result follows by noting that
\eq{
\lim_{n\to\infty}  
\frac{4\log_+^{\frac{3}{2}}(n/2)}{c \log(n) \sqrt{\pi \log\left(\frac{cn}{2\log_+^{\frac{3}{2}}(n/2)}\right)}} = \frac{4}{c \sqrt{\pi}}
}
and then choosing $n_0$ sufficiently large.
Moving to the second part. 
Let $\eta \in (0, 1]$ be a constant to be chosen later and $t_k = (1 + \eta)^k$ and 
\eq{
\alpha_k &= \frac{2\log_+^{\frac{1}{2}}\left(\frac{n}{2t_k}\right)}{c} &
K &= \max\set{k : t_k \leq 2en \quad \text{and} \quad \sqrt{\frac{n}{t_k}} \geq \alpha_k \geq 1}\,.
}
An easy computation shows that there exists a universal constant $c' > 0$ such that $t_K \geq c' n$. 
Then
\eqn{
&\P{\exists t \leq n : S_t \geq t\Delta + \max\set{0, \sqrt{2t \log\left(\frac{c n}{2t \log_+^{\frac{1}{2}}(n/(2t))}\right)}}} \nonumber \\ 
&\leq \sum_{k=0}^K \P{\exists t \leq t_k : S_t \geq \frac{t_k \Delta}{(1 + \eta)} + \sqrt{\frac{2t_k}{1+\eta} \log\left(\frac{n}{t_k\alpha_k}\right)}} + \P{\exists t_K \leq t \leq n : S_t \geq t\Delta}\,. \label{eq:conc1}
}
The second term in \cref{eq:conc1} is easily bounded by a peeling argument.
\eqn{
\P{\exists t_K \leq t \leq n: S_t \geq t\Delta} 
&\leq \sum_{i=1}^\infty \P{\exists t \leq (i+1)t_K : S_t \geq i t_K \Delta} \nonumber \\
&\leq \sum_{i=1}^\infty \exp\left(-\frac{i^2 t_K^2 \Delta^2}{2(i+1)t_K}\right) \nonumber \\
&\leq \frac{c''}{n\Delta^2}\,, \label{eq:conc2}
}
where $c'' > 0$ is some sufficiently large universal constant and we have used the fact that $t_K \geq c'n$ and naive algebra.
The first term in \cref{eq:conc1} is slightly more involved.
Using peeling and Lemmas \ref{lem:maximal} and \ref{lem:tech1} we have universal constants $c'$ and $c''$ such that
\eq{
&\sum_{k=0}^K \P{\exists t \leq t_k : S_t \geq \frac{t_k \Delta}{(1 + \eta)} + \sqrt{\frac{2t_k}{1+\eta} \log\left(\frac{n}{t_k\alpha_k}\right)}} \\
&\sr{(a)}\leq \sum_{k=0}^K \sqrt{\frac{1 + \eta}{\pi \log\left(\frac{n}{t_k \alpha_k}\right)}} \left(\frac{t_k \alpha_k}{n}\right)^{1/(1+\eta)} \exp\left(-\frac{t_k \Delta^2}{2(1 +\eta)^2}\right) \\
&\sr{(b)}\leq \sum_{k=0}^K \alpha_k \sqrt{\frac{2}{\pi \log\left(\frac{n}{t_k\alpha_k}\right)}} \left(\frac{t_k}{n}\right)^{1/(1+\eta)} \exp\left(-\frac{t_k \Delta^2}{2(1 +\eta)^2}\right) \\
&\sr{(c)}\leq \sum_{k=0}^K \frac{2\log_+^{\frac{1}{2}}\left(\frac{n}{t_k}\right)}{c} \sqrt{\frac{4}{\pi \log\left(\frac{n}{t_k}\right)}} \left(\frac{t_k}{n}\right)^{1/(1+\eta)} \exp\left(-\frac{t_k \Delta^2}{2(1 +\eta)^2}\right) \\
&\sr{(d)}= \frac{4}{c\sqrt{\pi}} \sum_{k=0}^K \left(\frac{t_k}{n}\right)^{1/(1+\eta)} \exp\left(-\frac{t_k \Delta^2}{2(1 +\eta)^2}\right) \\
&\sr{(e)}\leq \frac{3\cdot 4}{c\eta} \left(\frac{2(1 + \eta)}{n\Delta^2}\right)^{1/(1 + \eta)} \\
&\sr{(f)}\leq \frac{4 \cdot 3\cdot 4}{c\eta} \left(\frac{1}{n\Delta^2}\right)^{1/(1 + \eta)}\,.
}
where (a) follows from Lemma \ref{lem:maximal},
(b) since $\eta \in (0,1]$ and $\alpha_k \geq 1$,
(c) by substituting the value of $\alpha_k$ and because $\sqrt{n/t_k} \geq \alpha_k$,
(d) and (f) are trivial while (e) follows from Lemma \ref{lem:tech1}.
Then choose
\eq{
\eta = \frac{1}{\logp(n\Delta^2)} \in (0, 1]\,,
}
which satisfies
\eq{
\frac{1}{\eta} \left(\frac{1}{n\Delta^2}\right)^{1/(1+\eta)} \leq \frac{e}{n\Delta^2} \logp(n\Delta^2)\,. 
}
The result by combining the above reasoning with \cref{eq:conc2} and substituting into \cref{eq:conc1}.

\section{Proof of \cref{thm:asymptotic}}\label{sec:thm:asymptotic}
As before I assume for convenience that $\mu_1 \geq \mu_2 \geq \ldots \geq \mu_d$.
Let $n$ be fixed and abbreviate $\gamma_i(t) = \gamma(\nu_i(t-1), \sigma^2_i(t-1), n - t + 1)$.
Like the proof of \cref{thm:finite} there are two main difficulties. First, showing 
for suboptimal arms $i$ that $\gamma_i(t) < \mu_i + \Delta_i / 2$ occurs with sufficiently high probability once $T_i(t)$ is sufficiently large. 
Second, showing that $\gamma_1(t) \geq \mu_1 - \Delta_i / 2$ occurs with sufficiently high probability. 
Let $F$ be the event that there exists an arm $i$ and round $t \in \set{1,\ldots,n}$ such that
\eq{
\left|\hat \mu_i(t) - \mu_i\right| \geq \sqrt{\frac{2}{T_i(t)} \log (2dn^2)}\,.
}
As in the proof of \cref{thm:finite} we have $\P{F} \leq 1/n$.
The regret due to $F$ occurring is negligible, so from now on assume that $F$ does not hold.
By \cref{eq:update} we have
\eq{
\nu_i(t) = \eta_i(t) \nu_i + (1 - \eta_i(t)) \hat \mu_i(t)\,,
}
where
\eq{
\eta_i(t) = \frac{1}{1 + \sigma_i^2 T_i(t)} = O\left(\frac{1}{T_i(t)}\right)\,.
}
Therefore for sufficiently large $n$ and $u_i = \ceil{\frac{32}{\Delta^2} \log (dn^2)}$ and $T_i(t-1) = u_i$ 
\eq{
\gamma_i(t) 
&\leq \eta_i(t-1) \nu_i + (1 - \eta_i(t-1)) \hat \mu_i(t-1) + \sqrt{2 \sigma^2_i(t-1) \log (2n)} \\
&\leq \eta_i(t-1) (\nu_i - \mu_i) + \mu_i + \sqrt{\frac{2}{T_i(t-1)} \log (2n)} + \sqrt{\frac{2}{T_i(t-1)} \log (2dn^2)} \\
&\leq \eta_i(t-1) (\nu_i - \mu_i) + \mu_i + \Delta_i / 2 
\leq \mu_i + 3\Delta_i / 4\,.
}
For the first arm we have for $t \leq n /2$ that
\eq{
\gamma_1(t) 
\geq \eta_1(t-1) \nu_1 + (1 - \eta_1(t-1)) \hat \mu_1(t-1) + \sqrt{2\sigma_1^2(t-1) \log \beta_t}\,.
}
where by \cref{thm:gittins} 
\eq{
\beta_t \geq c \min\set{\frac{n/2}{\log_{+}^{\frac{3}{2}}(n/2)}, \,\, \frac{n\sigma_1^2(t-1)/2}{\log_{+}^{\frac{1}{2}}(n\sigma_1^2(t-1)/2)}}\,.
}
For any $\Delta > 0$ define $F_\Delta$ to be the event that there exists a $t \leq n/2$ for which
\eq{
\hat \mu_1(t-1) + \sqrt{\frac{2}{T_1(t-1)} \log \beta_t} \leq \mu_1 - \Delta\,.
}
By Lemma \ref{lem:conc} there is a universal constant $c'>0$ such that 
\eq{
\P{F_\Delta} \leq c'(\log(n)+ \logp(n\Delta^2))/(n\Delta^2)
}
and when $F_\Delta$ does not occur we have for $t \leq n/2$ that
\eq{
\gamma_1(t) 
&\geq \eta_1(t-1) \nu_1 + (1 - \eta_1(t-1)) \left(\mu_1 - \Delta - \sqrt{\frac{2}{T_1(t-1)} \log \beta_t}\right) \\ 
&\qquad  + \sqrt{2\sigma_1^2(t-1) \log \beta_t} \\
&\geq \mu_1 - \Delta - O\left(\frac{1}{T_1(t-1)}\right) - (1 - \eta_1(t-1)) \sqrt{\frac{2}{T_1(t-1)} \log \beta_t} \\
&\qquad  + \sqrt{2\sigma_1^2(t-1) \log \beta_t}\,. 
}
Comparing the ratio of the last two terms gives
\eq{
\frac{(1 - \eta_1(t-1)) \sqrt{\frac{2}{T_1(t-1)} \log \beta_t}}{\sqrt{2\sigma_1^2(t-1)} \log \beta_t}
&= \sqrt{\frac{\sigma_1^2 T_1(t-1)}{1 + \sigma_1^2 T_1(t-1)}} < 1\,.
}
Since $\log \beta_t = \Theta(\log n)$ for $t \leq n/2$, if $F_\Delta$ does not hold and $n$ is sufficiently large, then
\eq{
\gamma_1(t) \geq \mu_1 - 2\Delta \qquad \text{for all } t \leq n/2\,.
}
The proof is completed by duplicating the argument given in the proof of \cref{thm:finite} and taking the limit as $n$ tends to infinity.

\section{Proof of \cref{thm:bayes-arm}}\label{sec:thm:bayes-arm}

\begin{lemma}\label{lem:cgaussian}
Let $X \sim \mathcal N(0, \sigma^2)$ and $A$ be an arbitrary event.
Then
\eq{
\E \ind{A} X \leq \P{A} \sqrt{2\sigma^2 \log \left(\frac{1}{\P{A}}\right)}\,.
}
\end{lemma}

\begin{lemma}\label{lem:brownian}
Let $B_t$ be the standard Brownian motion.
Let $\epsilon > 0$ and define stopping time $\tau = \min\set{t : B_t = -\epsilon}$.
Then 
\eq{
\P{\tau > t} = \Phi\left(\frac{\epsilon}{\sqrt{t}}\right) - \Phi\left(\frac{-\epsilon}{\sqrt{t}}\right)\,. 
}
\end{lemma}

The proofs of Lemmas \ref{lem:cgaussian} and \ref{lem:brownian} are omitted, but the former follows by letting $\delta = \P{A}$ and noting that
the worst-case occurs when $A = \ind{X \geq \alpha}$ where $\alpha$ is such that 
\eq{
\int^\infty_\alpha \frac{\exp\left(-x^2/(2\sigma^2)\right)}{\sqrt{2\pi\sigma^2}} dx = \delta\,.
}
The second lemma is well known \citep[for example]{Ler86}.
For what follows we need an alternative view of the Gittins index.
Let $\nu \in \R$ and $1 \geq \sigma^2 \geq 0$. Furthermore define
\eq{
\sigma^2_t = \frac{\sigma^2}{1 + (t-1)\sigma^2} \cdot \frac{\sigma^2}{1 + t\sigma^2}\,,
}
which satisfies $\sum_{t=1}^\infty \sigma^2_t = \sigma^2$. I abbreviate $\sigma^2_{\leq t} = \sum_{s=1}^t \sigma^2_t$.
Let $\nu_1 = \nu$ and $\nu_{t+1} = \nu_t + \eta_t$ where $\eta_t \sim \mathcal N(0, \sigma^2_t)$.
Then the Gittins index is given by
\eq{
\gamma(\nu, \sigma^2, n) = \sup_{1 \leq \tau \leq n} \frac{\E\left[\sum_{t=1}^\tau \nu_t\right]}{\E[\tau]}\,,
}
where $\ind{\tau = t}$ is measurable with respect to the $\sigma$-algebra generated
by $\nu_1,\nu_2,\ldots,\nu_{t+1}$.
This is equivalent to the definition in \cref{eq:gittins}, but written in terms of the evolution of the posterior mean rather
than the observed reward sequences (these can be computed from each other in a deterministic way).

\begin{lemma}\label{lem:lower-weak}
There exists a universal constant $1 \geq c > 0$ such that
$\displaystyle \gamma(\nu, \sigma^2, n) \geq c \sqrt{2\sigma^2_{\leq n}}$.
\end{lemma}

The proof is left as an exercise (hint: choose a stopping time $\tau = \ceil{n/2}$ unless $\nu_{\ceil{n/2}}$ is large enough when $\tau = n$).
Recall that $\tau = \tau(\nu, \sigma^2, n)$ is the stopping rule defining the Gittins index for mean $\nu$, variance $\sigma^2$ and
horizon $n$. The following lemma shows that $\tau = n$ with reasonable probability and is crucial for the proof of \cref{thm:bayes-arm}.

\begin{lemma}\label{lem:tau-bound}
Let $\nu \in \R$ and $\sigma^2 > 0$ and let $\tau = \tau(\nu, \sigma^2, n)$. Then there exists a universal constant $c' > 0$ such that
\eq{
\P{\tau = n} \geq \frac{c'}{n^2 \log(n)}\,.
}
\end{lemma}

\begin{proof}
Assume without loss of generality (by translation and Lemma \ref{lem:shift}) that $\gamma(\nu, \sigma^2, n) = 0$ and let 
$\mathcal E = -\nu \geq c\sqrt{2\sigma^2_{\leq n}}$ by Lemma \ref{lem:lower-weak}.
By definition of the Gittins index we have
\eq{
\E\left[\sum_{t=1}^\tau \nu_t\right] = 0\,.
}
Let $\delta_t = \P{\tau \geq t \text{ and } \nu_t \geq \mathcal E/(2n)}$ and
$\alpha_t = \E[\nu_t|\tau \geq t \text{ and } \nu_t \geq \mathcal E/(2n)]$.
By Lemma \ref{lem:cgaussian} we have $\alpha_t \leq \sqrt{2\sigma_{\leq t}^2 \log(1/\delta_t)} - \mathcal E \leq \sqrt{2\sigma_{\leq t}^2 \log(1/\delta_t)}$.
Therefore
\eq{
0 
= \E\left[\sum_{t=1}^\tau \nu_t\right]
\leq -\mathcal E/2 + \sum_{t=2}^n \delta_t \alpha_t\,.
}
Therefore there exists a $t$ such that $\delta_t \alpha_t \geq \mathcal E / (2n)$ and so
\eq{
\delta_t \sqrt{2\sigma^2_{\leq t} \log\frac{1}{\delta_t}} \geq \delta_t \alpha_t \geq \frac{\mathcal E}{2n}
\geq \frac{c \sqrt{2\sigma^2_{\leq n}}}{2n}\,.
}
Straightforward analysis shows that
\eq{
\delta_t \geq \frac{c}{2n \log\left(\frac{2n}{c}\right)}\,.
}
Now we apply Lemma \ref{lem:brownian}. First note that if $\tau \geq t$ and $\nu_s \geq 0$ for all $s \geq t$, then $\tau = n$.
Therefore
\eq{
\P{\tau = n} 
&\geq \P{\tau \geq t \text{ and } \nu_s \geq 0 \text{ for all } s \geq t} \\
&\geq \frac{c}{2n \log\left(\frac{2n}{c}\right)} \left(\Phi\left(\frac{\mathcal E}{2n \sigma_{\leq n}}\right) - \Phi\left(-\frac{\mathcal E}{2n \sigma_{\leq n}}\right) \right) \\
&\geq \frac{c}{2n \log\left(\frac{2n}{c}\right)} \left(\Phi\left(\frac{c}{2n}\right) - \Phi\left(-\frac{c}{2n}\right) \right) \\
&\geq \frac{c^2}{8n^2 \log\left(\frac{2n}{c}\right)}
}
as required.
\end{proof}

\begin{remark}
It is rather clear that the above proof is quite weak. Empirically it appears that $\P{\tau = n} = \Omega(1/n)$, but the result above is sufficient 
for our requirements.
\end{remark}

\begin{proofof}{\cref{thm:bayes-arm}}
Let $\pi^*$ be the Bayesian optimal policy, which by assumption chooses $I_1 = 1$. 
Define $\pi$ be the policy that chooses arm $2$ until $t$ such that $\gamma_2(t) < \gamma_2(1)$ and thereafter
follows $\pi^*$ as if no samples had been observed and substituting the observed values for the second arm as it is chosen.
Let $\tau = \tau(\nu_2, \sigma^2_2, n)$, then $\pi$ chooses the second time until at least $t = \tau$.
For any policy $\pi$ let $T^\pi_i(n)$ be the number of plays of arm $i$ by policy $\pi$, which is a random variable.
Let $\nu_{i,s}$ be the posterior mean of the $i$th arm after $s-1$ plays. Then the value of policy $\pi$ is
\eq{
V^\pi = \E\left[\sum_{i=1}^d \sum_{s=1}^{T^\pi_i(n)} \nu_{i,s}\right]\,,
}
which is maximised by the Bayesian optimal policy $\pi^*$.
Therefore if we abbreviate $T^{\pi^*}_i(n)$ by $T_i(n)$ and $\gamma_2(1)$ by $\gamma$, then
\eqn{
\nonumber 0 
\nonumber &\geq V^\pi - V^{\pi^*} \\ 
\nonumber &\geq \E \left[\ind{\tau > T_2(n)} \sum_{s=T_1(n) - \tau + T_2(n)}^{T_1(n)} (\gamma - \nu_{1,s})\right] \\
\nonumber &\geq \sum_{s=1}^n \E\left[\ind{\tau - T_2(n) \geq s} (\gamma - \max_{s \leq n} \nu_{1,n})\right] \\
\nonumber &\geq \sum_{s=1}^n \P{\tau - T_2(n) \geq s} \left(\gamma - \nu_1 - \sqrt{2\sigma_1^2 \log \left(\frac{1}{\P{\tau - T_2(n) \geq s}}\right)}\right) \\
&= \sum_{s=1}^n \delta_s \left(\gamma - \nu_1 - \sqrt{2\sigma_1^2 \log \left(\frac{1}{\delta_s}\right)}\right)\,,
\label{eq:bayes-arm} 
}
where $\delta_s = \P{\tau - T_2(n) \geq s}$.
Since the optimal policy chooses action $1$ in the first round by assumption, if $\tau = n$, then $\tau - T_2(n) \geq 1$ is guaranteed.
Therefore by Lemma \ref{lem:tau-bound} we have 
\eq{
\delta_1 \geq \frac{c'}{n^2 \log(n)}\,.
}
By rearranging \cref{eq:bayes-arm} we have for any $\delta \leq \delta_1$ that
\eq{
\gamma 
&\leq \nu_1 + \frac{\sum_{s=1}^n \delta_s \sqrt{2\sigma_1^2 \log \left(\frac{1}{\delta_s}\right)}}{\sum_{s=1}^n \delta_s} \\
&\leq \nu_1 + \sqrt{2\sigma_1^2 \log\left(\frac{1}{\delta}\right)} + \frac{\sum_{s:\delta_s < \delta} \delta_s \sqrt{-2\sigma_1^2 \log \delta_s}}{\delta_1} \\ 
&\leq \nu_1 + \sqrt{2\sigma_1^2 \log\left(\frac{1}{\delta}\right)} + \frac{\sum_{s:\delta_s < \delta} \sqrt{2\sigma_1^2 \delta_s}}{\delta_1} \\
&\leq \nu_1 + \sqrt{2\sigma_1^2 \log\left(\frac{1}{\delta}\right)} + \frac{n\sqrt{2\sigma_1^2 \delta}}{\delta_1}\,.
}
The result is completed by choosing $\delta = \delta_1^2 n^{-2}$.  
\end{proofof}

\section{Technical Lemmas}\label{sec:tech}

\begin{lemma}\label{lem:tech1}
Let $\eta \in (0,1)$ and $\Delta > 0$. Then
\eq{
\sum_{k=0}^\infty (1 + \eta)^{\frac{k}{1 + \eta}} \exp\left(-(1 + \eta)^k \Delta^2\right)
\leq \frac{3}{\eta} \left(\frac{1}{\Delta^2}\right)^{1/(1 + \eta)}\,.
}
\end{lemma}

\begin{proof}
Let $L = \min\set{L : (1 + \eta)^k \geq \frac{1}{\Delta^2}}$. Then
\eq{
\sum_{k=0}^\infty &(1 + \eta)^{\frac{k}{1 + \eta}} \exp\left(-(1 + \eta)^k \Delta^2\right) \\
&=\sum_{k=0}^{L-1} (1 + \eta)^{\frac{k}{1 + \eta}} \exp\left(-(1 + \eta)^k \Delta^2\right) 
    + \sum_{k=L}^{\infty} (1 + \eta)^{\frac{k}{1 + \eta}} \exp\left(-(1 + \eta)^k \Delta^2\right) \\
&\leq\sum_{k=0}^{L-1} (1 + \eta)^{\frac{k}{1 + \eta}}
    + \sum_{k=L}^{\infty} (1 + \eta)^{\frac{k}{1 + \eta}} \exp\left(-(1 + \eta)^k \Delta^2\right) \\
&\leq \frac{1}{\eta} \left(\frac{1}{\Delta^2}\right)^{1/(1 + \eta)}
    + \sum_{k=L}^{\infty} (1 + \eta)^{\frac{k}{1 + \eta}} \exp\left(-(1 + \eta)^k \Delta^2\right) \\
&\leq \frac{1}{\eta} \left(\frac{1}{\Delta^2}\right)^{1/(1 + \eta)}
    + \left(\frac{1+\eta}{\Delta^2}\right)^{1/(1+\eta)} \sum_{k=0}^{\infty} (1 + \eta)^{\frac{k}{1 + \eta}} \exp\left(-(1 + \eta)^k \right) \\
&\leq \frac{1}{\eta} \left(\frac{1}{\Delta^2}\right)^{1/(1 + \eta)}
    + \left(\frac{1+\eta}{\Delta^2}\right)^{1/(1+\eta)} \sum_{k=0}^{\infty} (1 + \eta)^k \exp\left(-(1 + \eta)^k \right) \\
&\leq \frac{1}{\eta} \left(\frac{1}{\Delta^2}\right)^{1/(1 + \eta)}
    + \left(\frac{1+\eta}{\Delta^2}\right)^{1/(1+\eta)} \left(\frac{1}{e} + \int^\infty_0 (1 + \eta)^k \exp\left(-(1 + \eta)^k \right) dk\right) \\
&\leq \frac{1}{\eta} \left(\frac{1}{\Delta^2}\right)^{1/(1 + \eta)}
    + \left(\frac{1+\eta}{\Delta^2}\right)^{1/(1+\eta)} \frac{1}{e}\left(1 + \frac{1}{\log (1 + \eta)}\right) \\
&\leq \frac{2}{\eta} \left(\frac{1+\eta}{\Delta^2}\right)^{1/(1 + \eta)} \\
&\leq \frac{3}{\eta} \left(\frac{1}{\Delta^2}\right)^{1/(1 + \eta)}
}
as required.
\end{proof}

\section{Completing Proof of \cref{thm:gittins}}\label{sec:algebra}

First we need and easy lemma.

\begin{lemma}\label{lem:beta}
$f(\beta) = \frac{1}{\beta} \sqrt{\frac{1}{2\pi}} - \sqrt{\frac{\log \beta}{2}} \erfc\left(\sqrt{\log \beta}\right)$ satisfies:
\begin{enumerate}
\item $\lim_{\beta \to \infty} f(\beta) \beta \log(\beta) = \frac{1}{\sqrt{8\pi}}$.
\item $f(\beta) \geq \frac{1}{10\beta \log(\beta)}$ for all $\beta \geq 3$.
\item $f(\beta) \leq \frac{1}{\sqrt{8\pi}} \cdot \frac{1}{\beta \log \beta}$.
\end{enumerate}
\end{lemma}

\newcommand{\betaone}{\frac{m}{c\log^{\frac{3}{2}}_+(m)}}
\newcommand{\betatwo}{\frac{m\sigma^2}{c \log^{\frac{1}{2}}_+(m\sigma^2)}}

\subsubsect{Completing \cref{thm:gittins}}
To begin we choose $c = (40c''')^4$.
Recall that we need to show that
\eq{
m\sigma f(\beta) &\geq  
c''' \max\set{\sigma, \sqrt{\sigma^2 \log(\beta)}, \frac{1}{\sqrt{\sigma^2 \log(\beta)}}, \sqrt{mW\left(m\sigma^2 \log(\beta)\right)} \beta^{-\frac{7}{8}}}\,,
}
where
\eq{
\beta_1 &= \frac{m}{c \log_+^{\frac{3}{2}}(m)} &
\beta_2 &=\frac{m\sigma^2}{c \log_+^{\frac{1}{2}}(m\sigma^2)} &
\beta &= \min\set{\beta_1,\, \beta_2}
}
and $\beta \geq 3$ is assumed.
This is equivalent to showing:
\eq{
(1) \qquad m\sigma f(\beta) &\geq c'''\sigma  && \text{and}\\ 
(2) \qquad m\sigma f(\beta) &\geq c'''\sqrt{\sigma^2 \log(\beta)} && \text{and} \\
(3) \qquad m\sigma f(\beta) &\geq c'''/\sqrt{\sigma^2 \log(\beta)} && \text{and}  \\
(4) \qquad m\sigma f(\beta) &\geq c''' \sqrt{mW(m\sigma^2 \log(\beta))} \beta^{-\frac{7}{8}}\,.
}
By \cref{lem:beta} we have
\eq{
m\sigma f(\beta) &\geq \frac{m\sigma}{10 \beta \logp(\beta)} 
\geq \frac{m \sigma}{10 \cdot \betaone \logp\left(\betaone\right)}  \\
&\geq c''' \sigma \sqrt{\logp(m)} \geq c''' \sqrt{\sigma^2 \log(\beta)} \geq c''' \sigma\,.
}
Therefore (1) and (2) hold. For (3) we have
\eq{
m\sigma f(\beta) 
&\geq \frac{m\sigma}{10\beta \logp(\beta)} 
\geq \frac{m \sigma}{10\cdot \betatwo \log^{\frac{1}{2}}_+\left(\betatwo\right) \sqrt{\logp(\beta)}} \\
&\geq \frac{c'''}{\sqrt{\sigma^2 \log(\beta)}}\,.
}
Therefore (3) holds. Finally 
\eq{
m\sigma f(\beta) \beta^{\frac{7}{8}}
&\sr{(a)}\geq \frac{m\sigma}{10\beta^{\frac{1}{8}} \logp(\beta)}
\sr{(b)}\geq \frac{m\sigma}{40\beta^{\frac{1}{4}}} 
\sr{(c)}\geq \frac{m\sigma}{40 \left(\frac{m\sigma^2}{(40 c''')^4}\right)^{\frac{1}{4}}}
\sr{(d)}= c'''\sqrt{m} \left(m\sigma^2\right)^{\frac{1}{4}} \\
&\sr{(e)}\geq c'''\sqrt{m W(m\sigma^2 \log(m\sigma^2))} 
\sr{(f)}\geq c''' \sqrt{m W(m\sigma^2 \log(\beta))} \,,
}
where (a) follows from Lemma \ref{lem:beta}.
(b) since $\beta^{\frac{1}{8}} \log(\beta) \leq 4 \beta^{\frac{1}{4}}$ for $\beta \geq 3$.
(c) by substituting and naively upper bounding $\beta \leq \beta_2$.
(d) is trivial.
(e) since $x^{\frac{1}{4}} \geq \sqrt{W(x \log(x))}$ for all $x \geq 0$.
(f) is true because $\beta \leq \beta_2 \leq m\sigma^2$.

\section{A Brief History of Multi-Armed Bandits}\label{sec:history}

This list is necessarily not exhaustive. It contains in (almost) chronological order the papers that are most relevant for the present work, with a special
focus on earlier papers that are maybe less well known. Please contact me if there is a serious omission in this list. 

\begin{itemize}
\item \cite{Tho33} proposes a sampling approach for finite-horizon undiscounted Bernoulli bandits now known as Thompson sampling. No theoretical
results are given.
\item \cite{Rob52} sets out the general finite-horizon bandit problem and gives some Hannan consistent algorithms and general discussion. Robbins seems
to have been unaware of Thompson's work.
\item \cite{BJK56} present the optimal solution for the Bayesian one-armed bandit in the finite-horizon undiscounted setting. They also
prove many of the counter-intuitive properties of the Bayesian optimal solution for multi-armed bandits. The index used in the present article could just
as well be called the Bradt--Johnson--Karlin index.
\item \cite{Git79} presents the optimal solution for the Bayesian multi-armed bandit in the infinite-horizon geometrically discounted setting.
Gittins seems to have been unaware of the work by \cite{BJK56} who defined a similar index. Gittins' result does not apply in the
finite-horizon setting.
\item \cite{Whi80,Web92,Tsi94} all give alternative proofs and/or characterisations of the Gittins index for the infinite-horizon discounted case.
\item \cite{Bat83} gives an asymptotic approximation of the infinite-horizon Gittins index (see also the work by \cite{Yao06}).
\item The textbook by \cite{BF85} summarises many of the earlier results on Bayesian multi-armed bandits. They also prove that geometric discounting
is essentially necessary for the Gittins index theorem to hold and give counter-examples in the Bernoulli case.
\item \cite{LR85} develop asymptotically optimal frequentist algorithms for the multi-armed bandits and prove the first lower bounds.
\item \cite{Agr95} and \cite{KR95} independently develop the UCB algorithm.
\item \cite{BK97} give asymptotic approximations for the finite-horizon Gittins index (more specific results in a similar vein are given by \cite{CR65}).
\item \cite{ACF02} proves finite-time guarantees for UCB. 
\item \cite{Nin11} presents some methods for computing the finite-horizon Gittins index in the discrete case, and also suggest that index algorithm
is a good approximation for the intractable Bayesian solution (no reference given).
\item \cite{KKM12} gave promising empirical results for the finite-horizon Gittins index strategy for the Bernoulli case, but did not study its
theoretical properties. They also incorrectly claim that the Gittins strategy is Bayesian optimal. Their article played a large role in motivating
the present work.
\end{itemize}

\section{Table of Notation}\label{sec:notation}

\noindent
\renewcommand{\arraystretch}{1.5}
\hspace{-0.3cm}
\begin{tabular}{p{2.8cm}p{10cm}}
$d$                                       & number of arms \\
$n$                                       & horizon \\
$t$                                       & current round \\
$\mu_i$                                   & expected return of arm $i$ \\
$\mu^*$                                   & maximum mean reward, $\mu^* = \max_i \mu_i$ \\
$\hat \mu_i(t)$                           & empirical estimate of return of arm $i$ after round $t$ \\
$T_i(t)$                                  & number of times arm $i$ has been chosen after round $t$ \\
$I_t$                                     & arm chosen in round $t$ \\
$X_t$                                     & reward in round $t$ \\
$\nu_i$                                   & prior mean of arm $i$ \\
$\sigma_i^2$                              & prior variance for mean of arm $i$ \\
$\Delta_i$                                & gap between the expected returns of the best arm and the $i$th arm \\ 
$\Delta_{\max}$                           & maximum gap, $\max_i \Delta_i$ \\
$\Delta_{\min}$                           & minimum non-zero gap, $\min \set{\Delta_i : \Delta_i > 0}$ \\
$\logp(x)$                                & $\max\set{1, \log(x)}$  \\
$\plog(x)$                                & product logarithm $x = W(x) exp(W(x))$ \\
$\diag(\sigma_1^2,\ldots,\sigma_d^2)$     & diagonal matrix with entries $\sigma_1^2,\ldots,\sigma_d^2$ \\
$\gamma(\nu, \sigma^2, m)$                & Gittins index for mean/variance $(\nu, \sigma^2)$ and horizon $m$ \\
$\tau(\nu, \sigma^2, m)$                  & stopping time that determines the Gittins index \\
$\erfc(x)$                                & complementary error function 
\eq{
\erfc(x) = \frac{2}{\sqrt{\pi}} \int^\infty_x \exp(-y^2) dy} \\
$\erf(x)$                                 & error function $\erf(x) = 1 - \erfc(x)$ \\
$\mathcal N(\mu, \sigma^2)$               & Gaussian with mean $\mu$ and variance $\sigma^2$ \\ 
$\Phi(x)$                                 & standard Gaussian cdf
\eq{
\Phi(x) = \frac{1}{\sqrt{2\pi}}\int^x_{-\infty} \exp\left(-\frac{t^2}{2}\right) dt
}
\end{tabular}

\end{document}